\documentclass{article}

\usepackage[preprint]{neurips_2025}

\usepackage[utf8]{inputenc} 
\usepackage[T1]{fontenc}    
\usepackage{hyperref}       
\usepackage{url}            
\usepackage{booktabs}       
\usepackage{amsfonts}       
\usepackage{nicefrac}       
\usepackage{microtype}      
\usepackage{xcolor}         
\usepackage{microtype}
\usepackage{graphicx}
\usepackage{booktabs} 
\usepackage{amsthm}
\usepackage{thmtools,thm-restate}
\usepackage{hyperref}
\usepackage{amsmath}
\usepackage{amssymb}
\usepackage{mathtools}
\usepackage{import}
\usepackage{algorithm}
\usepackage{algorithmic}
\usepackage{boxedminipage}
\usepackage{multirow}
\usepackage{pifont}
\usepackage{bm}
\usepackage{enumitem}
\usepackage{subfigure}
\usepackage[capitalize,noabbrev]{cleveref}

\theoremstyle{plain}
\newtheorem{theorem}{Theorem}[section]

\newtheorem{lemma}[theorem]{Lemma}

\theoremstyle{definition}
\newtheorem{definition}[theorem]{Definition}
\newtheorem{assumption}[theorem]{Assumption}
\newtheorem{approximation}[theorem]{Approximation}

\theoremstyle{remark}

\usepackage[textsize=tiny]{todonotes}

\title{Differentially private and decentralized randomized power method}

\author{%
  Julien Nicolas\\
  INSA Lyon and McGill University\\
  Lyon, France \\
  \texttt{julien.nicolas@insa-lyon.fr}
  \And
  César Sabater\\
  INSA Lyon\\
  Lyon, France\\
  \texttt{cesar.sabater@insa-lyon.fr}
  \And
  Mohammed Maouche\\
  INSA Lyon\\
  Lyon, France\\
  \texttt{mohammed.maouche@insa-lyon.fr}
  \And
  Mark Coates\\
  McGill University\\
  Montréal, Canada\\
  \texttt{mark.coates@mcgill.ca}
  \And
  Sonia Ben Mokhtar\\
  INSA Lyon\\
  Lyon, France\\
  \texttt{sonia.ben-mokhtar@insa-lyon.fr}
}
\begin{document}

\maketitle

\begin{abstract}
The randomized power method has gained significant interest due to its simplicity and efficient handling of large-scale spectral analysis and recommendation tasks. However, its application to large datasets containing personal information (e.g., web interactions, search history, personal tastes) raises critical privacy problems. This paper addresses these issues by proposing enhanced privacy-preserving variants of the method. First, we propose a variant that reduces the amount of the noise required in current techniques to achieve Differential Privacy (DP). More precisely, we refine the privacy analysis so that the Gaussian noise variance no longer grows linearly with the target rank, achieving the same $(\varepsilon,\delta)$‑DP guarantees with strictly less noise. Second, we adapt our method to a decentralized framework in which data is distributed among multiple users. The decentralized protocol strengthens privacy guarantees with no accuracy penalty and a low computational and communication overhead. Our results include the provision of tighter convergence bounds for both the centralized and decentralized versions, and an empirical comparison with previous work using real recommendation datasets.
\end{abstract}

\section{Introduction}
The randomized power method has emerged as an efficient and scalable tool for addressing large-scale linear algebra problems central to modern machine learning pipelines.  By constructing an orthonormal basis for a matrix’s range in near-linear time, the method scales seamlessly to practical large datasets. Its reliance on simple matrix products ensures compatibility with sparse data representations and enables efficient parallelization and hardware acceleration on GPUs and distributed architectures.

Beyond its simplicity, the method provides strong approximation guarantees and accelerates a wide spectrum of applications. It has been used for principal component analysis (PCA) \citep{journee2010generalized}, singular value decomposition (SVD) \citep{halko2011finding}, truncated eigen-decompositions \citep{yuan2013truncated}, and matrix completion \citep{feng2018faster}. Extensions have powered recommender systems (e.g., Twitter \citep{gupta2013wtf}, GF-CF \citep{shen2021powerful}, BSPM \citep{shen2021powerful}), PageRank-style ranking \citep{ipsen2005analysis}, PDE solvers \citep{greengard1997new}, or large-scale least-squares and linear-system solvers \citep{rokhlin2008fast}.

However, as these methods are integrated into large-scale machine learning systems involving personal data, protecting user privacy is paramount. Unfortunately, the standard randomized power method does not inherently provide privacy guarantees. While its output might seem less sensitive than the input data, there is no formal guarantee against inference of private information embedded in the data.

To address and quantify privacy leakages, Differential Privacy (DP) has emerged as a powerful framework that provides strong guarantees and mitigates potential privacy leaks of an algorithm, ensuring that the output of an algorithm reveals little about any individual record in the input. Several works have attempted to apply DP to the randomized power method. For example, \citet{hardt2014noisy, balcan2016improved} developed centralized DP variants of the power method, whereas \citet{wang2020principal, guo2024fedpower} investigated federated DP protocols that can be used when data is kept locally across multiple devices. Adjacent works explore DP versions of PCA in both centralized \citep{NEURIPS2022_c150ebe1} and federated~\citep{wang2020principal,briguglio2023federated} settings and give optimal convergence bounds under distributional assumptions.

Despite these advancements, existing approaches suffer from several limitations. First, their performance heavily depends on the number of singular vectors being computed \citep{hardt2014noisy, balcan2016improved, guo2024fedpower, NEURIPS2022_c150ebe1, wang2020principal}, which impacts both utility and privacy guarantees. Second, they are designed for centralized settings \citep{hardt2014noisy, NEURIPS2022_c150ebe1}, where a trusted curator is assumed to hold the data. Moreover, some methods  \citep{grammenos2020federated, NEURIPS2022_c150ebe1} make strong assumptions about the data distribution (e.g., sub-Gaussianity) which makes it harder to use these methods in practice. Some federated versions \citep{briguglio2023federated, hartebrodt2024federated} claim to guarantee privacy due to the federated setting, but it has been shown that decentralization does not offer privacy by design \citep{geiping2020inverting}. 

MOD-SuLQ \citep{chaudhuri2013near} and its federated and streaming PCA variants \citep{grammenos2020federated} offer strong guarantees but are specifically tailored for reconstructing the top principal component ($k=1$). These methods add noise directly to the covariance matrix, rely on direct, computationally costly exact singular value decompositions (SVD) and are tailored for settings for which the number of samples largely exceeds the dimensionality ($n \gg d$). In contrast, the randomized power method iteratively adds noise directly to the approximation of the singular vectors themselves rather than the input matrix. This strategy reduces the dependence on the dimensionality of the original data, making it computationally more efficient and scalable. Memory-limited, streaming PCA methods such as those proposed by \citet{mitliagkas2013memory} are optimized for sequential processing under memory constraints but lack privacy guarantees, requiring additional modifications.

Finally, no fully decentralized versions exist to our knowledge, making them unsuitable for decentralized environments (e.g., recommender systems and social networks), where data is partitioned across users/devices and communications are restricted to a predefined communication graph.  The previously introduced DP federated versions \citep{balcan2016improved, wang2020principal, guo2024fedpower} use public channel communication and therefore require local DP, which hinders convergence. 

\subsection{Contributions}
In this work, we propose a low-noise, privacy-preserving variant of the randomized power method whose approximation error remains bounded regardless of subspace dimension, significantly improving its applicability to practical datasets. We additionally propose a decentralized variant that strengthens privacy guarantees while maintaining the computational efficiency, and approximation guarantees of the centralized method. We summarize our contributions below:
\begin{itemize}[leftmargin=*]
  \item \textbf{An Improved Differentially Private Randomized Power Method.}
    We introduce a Differentially Private variant of the randomized power method under a generalized adjacency model, making it more suitable for practical applications. Unlike prior works, which relied on loose sensitivity bounds extrapolated from one-dimensional vectors, we derive a significantly tighter sensitivity analysis directly applicable to the multidimensional case. We show theoretically with refined convergence bounds and empirically with a recommender system use-case\footnote{We provide an anonymous code repository \href{https://anonymous.4open.science/r/neurips_submission-6C05/README.md}{here}.} that this refinement considerably reduces the amount of noise and the approximation error for a given privacy budget.

  \item \textbf{Decentralized private power method.}
    We propose a decentralized version of our algorithm, allowing settings for which data is distributed across multiple users or devices (e.g., recommender systems). This version is compatible with Secure Aggregation \citep{bell2020secure, bonawitz2017practical, kadhe2020fastsecagg}, or fully decentralized private averaging approaches using correlated noise \citep{sabater2022accurate, allouah2024privacy}. Our decentralized approach preserves the effectiveness of the centralized version while incorporating Distributed Differential Privacy.

  \item \textbf{Rigorous privacy guarantees.}
    We present a new, more straightforward proof of a privacy result for the differentially private randomized power method. This proof addresses some mistakes in a privacy proof from \citet{hardt2012beating,hardt2013beyond} that have been reproduced in several follow-up works \citep{hardt2014noisy, balcan2016improved} and allows for a wider range of DP parameters.
    
\end{itemize}

The remainder of the paper is organized as follows: In Section~\ref{sec:background} we review the necessary background and notation. Section~\ref{proposedbounds} introduces our generalized adjacency notion, derives the improved sensitivity bound (Theorem~\ref{improvedsensitivitybound}), and presents the overall privacy proof (Theorem~\ref{thm:privacyofppm}) together with runtime-dependent convergence guarantees (Theorem~\ref{thm:proposed_upper_bound}). In Section~\ref{runtimeindebounds} we turn this into a fully runtime-independent bound (Theorem~\ref{thm:ippmruninde}). Section~\ref{decentrversion} develops the decentralized variant (Algorithm~\ref{alg:decentralizedpowermethod}), proves its equivalence to the centralized version (Theorem~\ref{privconvdecalgo}), and analyzes its communication and computation overhead. We empirically compare the methods on standard recommendation datasets in Section~\ref{empiricalcomparison}, and conclude in Section~\ref{sec:conclusion} with a discussion of limitations and future directions.

\section{Background Material and Related Work}
\label{sec:background}
\paragraph{Matrix Norms and Notations.}
For any matrix $\bm{X}$, the element-wise maximum norm is defined as $\|\bm{X}\|_{\max} = \max_{i,j} |\bm{X}_{ij}|$, where $\bm{X}_{ij}$ is the $(i,j)$-th element of $\bm{X}$. The $\ell_2$-norm is denoted as $\|\bm{X}\|_2$ and the Frobenius norm as $\|\bm{X}\|_F$. We use $\bm{X}_{j :}$ to denote the the $j$-th row of $\bm{X}$.

\paragraph{Eigenvalue Decomposition.} Let $\bm{A} \in \mathbb{R}^{n \times n}$ be a real-valued positive semi-definite matrix, where $n$ is a positive integer. The eigenvalue decomposition of $\bm{A}$ is given by $\bm{A} = \bm{U} \bm{\Lambda} \bm{U}^{\top}, $ where $\bm{U} \in \mathbb{R}^{n\times n}$ is a matrix of eigenvectors and $\bm{\Lambda}  \in \mathbb{R}^{n\times n}$ is a diagonal matrix containing corresponding eigenvalues.

\paragraph{QR Decomposition.}
We will use the matrix QR decomposition, obtained using the Gram-Schmidt procedure. Given a matrix $\bm{X} \in \mathbb{R}^{n \times p}$, the QR decomposition factorizes it as $\bm{X} = \bm{Q}\bm{R}$, where $\bm{Q} \in \mathbb{R}^{n\times p}$ is an orthonormal matrix (\emph{i.e.}, $\bm{Q}^{\top} \bm{Q} = \bm{I}$) and $\bm{R} \in \mathbb{R}^{p \times p}$ is an upper triangular matrix.

\paragraph{Gaussian Random Matrices.}
We denote by $\mathcal{N}(\mu, \sigma^2)^{n \times p}$ a $(n \times p)$ random matrix where each element is an independent and identically distributed (i.i.d.) random variable according to a Gaussian distribution with mean $\mu$ and variance $\sigma^2$.

\paragraph{Coherence Measures of a Matrix.}
We define below two coherence measures for the matrix $\bm{A}$, which will allow us to give convergence bounds for our method.
\vspace{-3mm}
\begin{itemize}[leftmargin=*]
    \item The $\mu_0$-coherence of $\bm{A}$ is the maximum absolute value of its eigenvector entries, defined as $ \mu_0(\bm{A}) = \|\bm{U}\|_{\max} = \max_{i,j} |\bm{U}_{ij}|$.
    \item The $\mu_1$-coherence of $\bm{A}$ is the maximum row $\ell_2$-norm of its eigenvectors, defined as $\mu_1(\bm{A}) = \|\bm{U}\|_{2,\infty} = \max_{i} (\|\bm{U}_{i:}\|_2)$.
\end{itemize}

\subsection{Differential privacy}
With positive integers $n$ and $m$ specifying the matrix dimensions, let $D_1 \in \mathbb{R}^{n \times m}$ and $D_2 \in \mathbb{R}^{n \times m}$ be two matrices representing two datasets embedding sensitive information. $D_1$ and $D_2$ are said to be adjacent ($D_1 \sim D_2$) if they differ on one sensitive element of the dataset. For example, in a recommender system application, $D_1$ and $D_2$ can be binary user-item interaction matrices and a sensitive element of the dataset can be an entry in the matrix, corresponding to a user-item interaction.

A randomized algorithm $ \mathcal{M} $ provides ($\epsilon, \delta$)-Differential Privacy (DP) if for all adjacent datasets $D_1$ and $ D_2$, and for all measurable subsets $ S \subseteq \text{Range}(\mathcal{M}) $, the following holds:
\begin{align}
\Pr[\mathcal{M}(D_1) \in S] \leq e^{\epsilon} \Pr[\mathcal{M}(D_2) \in S] + \delta,
\end{align}
where $ \epsilon $ is a small positive scalar representing the privacy loss (smaller values indicate stronger privacy guarantees), and $ \delta $ is a (typically negligible) probability  that the privacy guarantee fails.  

Let \(f : \mathbb{R}^{n \times m} \to \mathbb{R}^d\) be a query function associated with a mechanism $ \mathcal{M} $. DP guarantees of mechanisms are defined using the sensitivity of $f$. In our contribution we use the $\ell_2$-sensitivity, denoted by $\Delta_2 (f)$ or $\Delta_2$ and defined as $\Delta_2 = \max_{D_1 \sim D_2} \| f(D_1) - f(D_2) \|_2\,.$

\subsection{Privacy-Preserving Randomized Power Method}
\label{subsec:privpowerit}

Let \(k\) be the target rank, let \(b>0\) be a small positive integer, and let \(\eta>0\) be an approximation tolerance. Given a positive semi-definite (PSD) matrix \(\bm{A}  \in \mathbb{R}^{n\times n}\), the aim of the randomized power method is to construct a matrix \(\bm{Q} \in \mathbb{R}^{n \times p}\), where \(p = k + b\), whose column space approximates that of the top \(k\) eigenvectors  of \(\bm{A}\), \textit{i.e.}, \(\bm{U}_k  \in \mathbb{R}^{n\times k}\). Specifically, it aims to satisfy
\begin{equation}
\label{eq:privpowermethod}
\bigl\|\bm{U}_k - \bm{Q}\bm{Q}^{\top}\bm{U}_k \bigr\| \le \eta.
\end{equation}

To protect sensitive information in \(\bm{A}\) while computing \(\bm{Q}\), prior work
\citep{hardt2014noisy,balcan2016improved} shows that the randomized power method can be implemented as Algorithm~\ref{alg:centralizedpowermethod} for it to satisfy \((\epsilon,\delta)\)-Differential Privacy with adjacency defined as a single element change in $\bm{A}$. The corresponding adjacency notion is defined in more detail in Equation~\eqref{pastupdatemodel}.

To compute the standard deviation of the noise required to satisfy  \((\epsilon,\delta)\)-DP, one needs to bound the \(\ell_2\)-sensitivity \(\Delta_l\) of the computations at each iteration of the algorithm.
Prior works~\citep{hardt2014noisy,balcan2016improved} use the estimate
\begin{equation}
    \label{priorbound}
      \Delta_{l}^{\text{prior}} \triangleq \sqrt{p}\|\bm{X}^l\|_{\max},
\end{equation}
which upper-bounds the true sensitivity, i.e.,\ \(\Delta_l \le \Delta_l^{\text{prior}}\). Here, \(\Delta_l\) bounds the change on
\(\bm{A}\bm{X}^{l-1}\) under a single-element perturbation in \(\bm{A}\).  

Using this bound\footnote{See Fig.~3 of \citet{hardt2014noisy}, Alg.~2 of \citet{balcan2016improved}, and Theorem~6
of \citet{guo2024fedpower}.}, the standard deviation \(\sigma_l\) of the Gaussian noise added at iteration \(l\) to achieve
\((\epsilon,\delta)\)-DP is $\sigma_l =\sqrt{p}\,\bigl\|\bm{X}^l\bigr\|_{\max}\epsilon^{-1}\sqrt{4L\ln(1/\delta)}\,,$ where \(L\) is the total number of power iterations.
\begin{algorithm}[ht]
	\begin{algorithmic}[1]
	\caption{Privacy-preserving randomized power method}
    \label{alg:centralizedpowermethod}
	\STATE{\textbf{Input}: Matrix $\bm A\in\mathbb R^{n\times n}$,  number of iterations $L$, target rank $k$, iteration rank $p\geq k$, 
		privacy parameters $\epsilon,\delta$}
	\STATE{\textbf{Output}: approximated eigen-space $\bm X^L\in\mathbb R^{n\times p}$, with orthonormal columns.}
	\STATE{\textbf{Initialization}}: orthonormal $\bm X^0\in\mathbb R^{n\times p}$ by QR decomposition on a random Gaussian matrix $\bm G_0$;
	noise variance parameter $\sigma=\epsilon^{-1}\sqrt{{4L\log(1/\delta)}}$;\\
	\FOR{$\ell=1$ to $L$}
	\STATE{
		1. Compute $\bm Y_\ell=\bm A\bm X^{\ell-1}+\bm G_\ell$ with $\bm G_\ell\sim \mathcal{N}(0,\sigma_l^2={\Delta}_l^2 \cdot \sigma^2)^{n\times p}$\\
		2. Compute QR factorization $\bm Y_\ell=\bm X^\ell\bm R_\ell$.
	}
	\ENDFOR
	\end{algorithmic}
\end{algorithm}

\subsection{Existing convergence bounds}
\label{existingbounds}
To our knowledge, the strongest convergence bound for the privacy-preserving randomized power method (Algorithm \ref{alg:centralizedpowermethod}) uses the bound $\Delta_l \leq \Delta_l^{prior}$ and is given by \citet{balcan2016improved} in their Corollary 3.1.
The following theorems rely on the conditions in Assumption~\ref{assum_check}, which are enforced by judicious choice of the added noise.

\begin{assumption}
\label{assum_check}
Let  $\bm A\in\mathbb R^{n\times n}$ be a symmetric matrix. Fix a target rank $k$, an intermediate rank $q\geq k$, and an iteration rank $p$, with $k\leq q\leq p$. Let $\bm U_q\in\mathbb R^{n\times q}$ be the 
	top-$q$ eigenvectors of $\bm A$ and let $\lambda_1\geq\cdots\geq\lambda_n\geq 0$ denote
	its eigenvalues. Let us fix
$ \eta = O\left(\frac{\lambda_q}{\lambda_k}\cdot\min\left\{\frac{1}{\log\left(\frac{\lambda_k}{\lambda_q}\right)},\frac{1}{\log\left(\tau n\right)}\right\}\right).$

	Assume that at every iteration $l$ of Algorithm 1, $\bm G_\ell$ satisfies, 	for some constant $\tau>0$:
	\begin{align}
	\|\bm G_\ell\|_2 &= O\left(\eta(\lambda_k-\lambda_{q+1})\right)\,, \quad \text{and}  \quad
	\|\bm U_q^\top\bm G_\ell\|_2 = O\left(\eta\left(\lambda_k - \lambda_{q+1}\right)\frac{\sqrt{p}-\sqrt{q-1}}{\tau\sqrt{n}}\right)\,.
	\end{align}
\end{assumption}

We now restate the Private Power Method major result in \citep{balcan2016improved}.
\begin{theorem}[Private Power Method (PPM), reduction to $s=1$ from the proof in Appendix C.1 from \citet{balcan2016improved}]\label{thm:dppca}
	Let  $\bm A\in\mathbb R^{n\times n}$ be a symmetric data matrix. Fix target rank $k$, intermediate rank $q\geq k$, and iteration rank $p$, with $2q\leq p\leq n$. Suppose the number of iterations $L$ is set as $L=\Theta(\frac{\lambda_k}{\lambda_k-\lambda_{q+1}}\log(n))$.
	Let $\epsilon,\delta\in(0,1)$ be the differential privacy parameters.
Let $\bm U_q\in\mathbb R^{n\times q}$ be the 
	top-$q$ eigenvectors of $\bm A$ and let $\lambda_1\geq\cdots\geq\lambda_n\geq 0$ denote
	its eigenvalues.
	Then Algorithm \ref{alg:centralizedpowermethod} with $\Delta_l = \Delta_l^{prior}$ is ($\epsilon$, $\delta$)-DP and with probability at least 0.9
    \begin{align}
        \| (\bm{I} - \bm{X}^L ({\bm{X}^L})^\top  ) \bm{U}_k\|_2 \leq \eta \quad \text{and} \quad   \| (\bm{I} - \bm{X}^L ({\bm{X}^L})^\top  ) \bm{A}\|_2^2 \leq \lambda_{k+1}^2 + \eta^2  \lambda_k^2
    \end{align}
	$$
	\text{    with } \quad \eta = O\left(\frac{\epsilon^{-1}\max_{l}{(\|\bm{X}^l\|_{\max})}\sqrt{4pLn\log(1/\delta)\log (L)}}{\lambda_k-\lambda_{q+1}}\right)
	$$
    $$
	\text{    and also } \quad \eta = O\left(\frac{\epsilon^{-1}\|\bm U\|_{\max}\sqrt{4pLn\log(1/\delta) \log (n)\log (L)}}{\lambda_k-\lambda_{q+1}}\right).
	$$
\end{theorem}
\section{Proposed Differentially Private Power Method Convergence Bounds}
\label{proposedbounds}
In this section, we present an enhanced Differentially Private Power Method. First, we introduce a generalized definition of adjacency that goes beyond single‐entry changes in a PSD matrix to allow for more applications. Then, we modify the Private Power Method to calibrate DP noise using a new, strictly tighter sensitivity bound, eliminating the $\sqrt p$ (associated to target rank) factor in prior work. We finish by showing analytically that this refinement yields strictly sharper convergence guarantees.

\paragraph{Adjacency notion.} In prior works \citep{hardt2014noisy, balcan2016improved}, two datasets represented by PSD matrices $\bm{A} \in \mathbb{R}^{n \times n}$ and $\bm{A}' \in \mathbb{R}^{n \times n}$ are considered adjacent (denoted $\bm{A} \sim \bm{A}'$) if they differ by a change in a single element, with a Frobenius norm of at most 1:
\begin{equation}
\label{pastupdatemodel}
    \bm{A}' = \bm{A} + c \cdot \bm{e}_i \bm{e}_j^\top,
\end{equation}
where $\bm{e}_i, \bm{e}_j \in \mathbb{R}^n$ are canonical basis vectors, $ c \leq 1 $ represents the magnitude of the change, and $ 0 \leq i, j < n $. This adjacency notion models a sensitive change as a modification to a single element in $\bm{A}$ to protect individual element-wise updates under Differential Privacy. 
However, since $\bm{A}$ is symmetric positive semi-definite and changes must preserve this property, this formulation restricts updates to be diagonal, hindering possible applications. 

We propose a new, more general, notion of adjacency to allow for other types of updates:
\begin{equation}
\label{proposedupdatemodel}
    \bm{A}' = \bm{A} + \bm{C},
\end{equation}
where $\bm{C} \in \mathbb{R}^{n \times n}$ is a symmetric matrix representing the update, subject to \label{conditionadjacency} $\sqrt{\sum_{i=1}^n \| \bm{C}_{i,:} \|_1^2} \leq 1$.

The proposed adjacency notion is strictly more general than Equation~\eqref{pastupdatemodel}. For example, setting $\bm{C} = c \cdot \bm{e}_i \bm{e}_j^\top$ recovers the original definition. $\bm{C}$ can have non-zero entries on the diagonal, anti-diagonal, or any symmetric pattern, allowing for a variety of updates maintaining symmetry.

In the context of recommender systems, where $\bm{A} = \bm{R}^\top \bm{R}$ represents the item-item similarity matrix and $\bm{R} \in \mathbb{R}^{m \times n}$ is the user-item interaction matrix, our proposed notion of adjacency enables element-wise modifications in $\bm{R}$ (\emph{i.e.}, protecting individual user-item interactions). Such changes in $\bm{R}$ propagate to multiple elements of $\bm{A}$, which could not be adequately accounted for under the previous adjacency definition (Equation~\eqref{pastupdatemodel}). By adopting our more general adjacency definition, we make our privacy guarantees applicable to a wider range of real-world scenarios.

\paragraph{Sensitivity bound.}
The previous sensitivity bound $\Delta_l^{\text{prior}}$  \citep{hardt2014noisy, balcan2016improved} for $ \bm{A}\bm{X}^{l-1} $ defined in Equation~\eqref{priorbound} was estimated by extrapolating from the case where $ \bm{X}^{l-1} \in \mathbb{R}^{n \times 1} $ to the general case $ \bm{X}^{l-1} \in \mathbb{R}^{n \times p} $, leading to a dependence on $\sqrt{p}$. This leads to an overestimation of the sensitivity and to unnecessarily large noise addition. We directly analyze the change in $ \bm{X}^{l-1} \in \mathbb{R}^{n \times p} $ and derive a strictly tighter bound on the sensitivity $\Delta_l$. By using our adjacency notion from Equation~\eqref{proposedupdatemodel}, we establish the following result (proof in Appendix~\ref{sec:app_proof1}):

\begin{theorem}[Improved Sensitivity Bound]
\label{improvedsensitivitybound}
Let $ \bm{A}' $ be defined as in Equation~\eqref{proposedupdatemodel}, and consider the sensitivity $\Delta_l = \sup_{\bm{A} \sim \bm{A}'}  \| \bm{A}' \bm{X}^l - \bm{A} \bm{X}^l \|_F$. Then, 
\begin{equation}
\label{upper_bound_max}
    \Delta_l \leq \max_i \| \bm{X}_{i:}^l \|_2 \triangleq \hat{\Delta}_l.
\end{equation}
\end{theorem}

\paragraph{Privacy proof.} Our algorithm achieves differential privacy by iteratively adding calibrated noise at each round of the power method. The total privacy guarantee across iterations is then derived using results from adaptive composition of DP mechanisms, as initially proposed in \citep{bun2016concentrated}, allowing us to precisely quantify the cumulative privacy loss across multiple iterative steps. To clarify the ambiguities or errors present in the previous works (see Appendix~\ref{note:note_proofs}), and ensure that our privacy guarantees are met, we propose a result with a new proof of the Differential Privacy guarantees for our overall algorithm in Theorem~\ref{thm:privacyofppm}, whose proof is in Appendix~\ref{sec:app_proof2}.

\begin{theorem}[Privacy proof for the PPM]
\label{thm:privacyofppm} 
Let $\delta\in(0,1)$ and $\epsilon>0$ such that $\delta \leq \exp{(-\frac{\epsilon}{4})}$. Then, Algorithm~\ref{alg:centralizedpowermethod} with $\Delta_l = \max_i \|\bm{X}_{i:}^l \|_2$ is ($\epsilon, \delta$)-Differentially Private.
\end{theorem}
\paragraph{Improved convergence bound.} 
Building on Equation~\eqref{upper_bound_max}, we present in Theorem~\ref{thm:proposed_upper_bound} a strictly tighter convergence bound than the one proposed in \citep{balcan2016improved}. Additionally, unlike past proofs \citep{hardt2014noisy, balcan2016improved}, our proposed privacy proof (given in Theorem~\ref{thm:privacyofppm}) does not restrict $\epsilon \leq 1$. We provide the proof in Appendix~\ref{proofupperboundindep}. 

\begin{theorem}[Improved PPM with Runtime-Dependent Bound]\label{thm:proposed_upper_bound}
	Let  $\bm A\in\mathbb R^{n\times n}$ be a symmetric data matrix. Fix target rank $k$, intermediate rank $q\geq k$ and iteration rank $p$ with $2q\leq p\leq n$. Suppose the number of iterations $L=\Theta(\frac{\lambda_k}{\lambda_k-\lambda_{q+1}}\log(n))$.
	Let $\delta\in(0,1)$ and $\epsilon>0$ be privacy parameters such that $\delta \leq \exp{(-\frac{\epsilon}{4})}$. Let $\bm U_k\in\mathbb R^{n\times k}$ be the 
	top-$k$ eigenvectors of $\bm A$ and let $\lambda_1\geq\cdots\geq\lambda_n\geq 0$ denote its eigenvalues. Then Algorithm \ref{alg:centralizedpowermethod} is ($\epsilon$,$\delta$)-DP with $\Delta_l = \max_i \|\bm{X}_{i:}^l \|_2$ and with probability at least 0.9
    \begin{align}
        \| (\bm{I} - \bm{X}^L ({\bm{X}^L})^\top  ) \bm{U}_k\|_2 \leq \eta \quad \text{and} \quad   \| (\bm{I} - \bm{X}^L ({\bm{X}^L})^\top  ) \bm{A}\|_2^2 \leq \lambda_{k+1}^2 + \eta^2  \lambda_k^2
    \end{align}
	\begin{align}
	\text{with } \quad \eta = O\left(\frac{ \epsilon^{-1} \max_{i,l} \|\bm{X}_{i:}^l \|_2 \sqrt{Ln\log(1/\delta)\log (L)}}{\lambda_k-\lambda_{q+1}}\right).
	\end{align}
\end{theorem}

\section{Proposed Runtime-Independent Convergence Bound}
We presented in Theorem~\ref{thm:proposed_upper_bound} a convergence bound involving $\max_i \| \bm{X}_{i:} \|_2$, which is only observable during the execution of the algorithm. To provide a more practical analysis, we now derive a runtime-independent convergence bound in Theorem~\ref{thm:ippmruninde} by careful bounding of $\max_i \| \bm{X}_{i:} \|_2$. We provide a proof in Appendix~\ref{sec:app_proof5}. This bound makes it possible to have a tight analysis in two regimes:
\vspace{-2mm}
\begin{itemize}[leftmargin=*]
    \item \textbf{Small $\mu_0$-coherence, small $p$}: Previously proposed in \citet{balcan2016improved}, this bound is useful in a regime with small $\mu_0$ when computing few eigenvectors, with a dependence on $\sqrt{p  \log({n})} \cdot \mu_0(\bm{A})$.
    
    \item \textbf{Larger $\mu_0$ or $p$}: We propose a new bound tailored for the multi-dimensional power method, depending on $\mu_1(\bm{A})$, with a reduced dependence on the number of eigenvectors $p$.
\end{itemize}
We note that we are likely to be in the second regime in practice, as we highlight in Section~\ref{empiricalcomparison}.
\label{runtimeindebounds}
\begin{restatable}{theorem}{ppmruntimeinde}\textbf{Improved PPM with Runtime-Independent Bound}\label{thm:ippmruninde}

	Let  $\bm A\in\mathbb R^{n\times n}$ be a symmetric data matrix. Fix target rank $k$, intermediate rank $q\geq k$ and iteration rank $p$ with $2q\leq p\leq n$. Suppose the number of iterations $L$ is set as $L=\Theta(\frac{\lambda_k}{\lambda_k-\lambda_{q+1}}\log(n))$. Let $\bm U_q\in\mathbb R^{n\times q}$ be the 
	top-$q$ eigenvectors of $\bm A$ and let $\lambda_1\geq\cdots\geq\lambda_n\geq 0$ denote its eigenvalues. Let $\delta\in(0,1)$ and $\epsilon>0$ be privacy parameters such that $\delta \leq \exp{(-\frac{\epsilon}{4})}$ .
	Then
 Algorithm \ref{alg:centralizedpowermethod} is ($\epsilon$,$\delta$)-DP and we have with probability at least 0.9
    \begin{align}
        \| (\bm{I} - \bm{X}^L ({\bm{X}^L})^\top  ) \bm{U}_k\|_2 \leq \eta \quad \text{and} \quad   \| (\bm{I} - \bm{X}^L ({\bm{X}^L})^\top  ) \bm{A}\|_2^2 \leq \lambda_{k+1}^2 + \eta^2  \lambda_k^2
    \end{align}
	\begin{align}
	\text{with } \quad \eta = O\left(\frac{ \epsilon^{-1} \cdot \min(\mu_0(\bm{A})\sqrt{p \cdot  \log({n})}, \mu_1(\bm{A}))  \cdot \sqrt{Ln\log(1/\delta)\log (L)}}{\lambda_k-\lambda_{q+1}}\right).
	\end{align}
\end{restatable}
\section{Decentralized version}
\label{decentrversion}
\begin{table*}[htb]
  \centering
  \caption{Per‐iteration overhead: Secure Aggregation vs. Federated Learning.}
  \label{tbl:overhead_compare}
  \begin{tabular}{lcccc}
    \toprule
      & \multicolumn{2}{c}{Secure Aggregation} 
      & \multicolumn{2}{c}{Federated Learning} \\
    \cmidrule(lr){2-3}\cmidrule(lr){4-5}
      & Client & Server & Client & Server \\
    \midrule
    Communication 
      & $O\bigl(\log(s) + n p\bigr)$ 
      & $O\bigl(s\,[\log(s) + n p]\bigr)$ 
      & $O(n p)$ 
      & $O\bigl(s\,n p\bigr)$ 
      \\
    Computation 
      & $O\bigl(\log^2(s) + n p\log(s)\bigr)$ 
      & $O\bigl(s\,[\log^2(s) + n p\log(s)]\bigr)$ 
      & $O(n p)$ 
      & $O\bigl(s\,n p\bigr)$ 
      \\
    \bottomrule
  \end{tabular}
  \vspace{-3mm}
\end{table*}
In this section, we consider a decentralized setting where the matrix $\bm{A}$ is distributed across multiple clients. Specifically, each client $i$ holds a private matrix $\bm{A}^{(i)} \in \mathbb{R}^{n \times n}$, such that the global matrix is the sum of these local matrices: $\bm{A} = \sum_{i=1}^s \bm{A}^{(i)}$. The goal is for the clients to collaboratively compute an orthonormal basis for the range of $\bm A$, similar to the centralized randomized power method, but without revealing their individual private matrices $\bm A^{(i)}$ to the server or to other clients.

The randomized power method involves linear operations, making it well-suited for parallelization and distributed computation. \citet{balcan2016improved, guo2024fedpower} proposed private and federated power methods using communication over public channels between clients and a server. However, these approaches rely on local Differential Privacy, since the data exchanged can be observed by everyone, which requires high levels of noise to ensure privacy.

To enhance privacy while retaining the benefits of distributed computation, we use Secure Aggregation, a lightweight Multi-Party Computation protocol. Secure Aggregation allows clients to collaboratively compute sums without revealing individual data, enabling the use of distributed DP. This approach offers privacy guarantees similar to central DP and eliminates the need for a trusted curator. Distributed DP has been extensively studied in the literature \citep{ goryczka2013secure, ghazi2019scalable, kairouz2021distributed,chen2021communication, wei2024distributed}.

In distributed DP, each client adds carefully calibrated noise to their local contributions before participating in the Secure Aggregation protocol. The noise is designed such that the sum across clients has a variance comparable to that used in central DP, thereby achieving similar Differential Privacy guarantees without requiring a trusted aggregator, and allowing clients to keep data locally.

We operate under the \textbf{honest-but-curious} threat model, where clients follow the protocol correctly but may attempt to learn information from received data. We also assume that there are no corrupted or dropout users during the computation. For simplicity, this papers neglects the effects of data quantization or errors introduced by modular arithmetic in distributed DP and refer to \citet{kairouz2021distributed} for an implementation taking this into account.

It is also possible to use a fully decentralized DP protocol to perform secure averaging without relying on a central server, as demonstrated by \citet{sabater2022accurate}. Adopting such approaches can further enhance decentralization while maintaining similar utility and communication costs for the method.
\subsection{Decentralized and private power method using distributed DP:}
We introduce a federated version of the Privacy-Preserving Power Method in Algorithm~\ref{alg:decentralizedpowermethod}. This version significantly reduces the noise variance by a factor of $s p \log n \, \| \bm{U} \|_{\infty}^2$ compared to the method in \citet{balcan2016improved}. The improvement comes from the fact that we emulate centralized Differential Privacy (DP) over secure channels, avoiding the higher noise required in local DP settings with public channels, and that we use our generally tighter sensitivity bound.

We now give a simplified definition of Secure Aggregation and demonstrate the equivalence between Algorithm~\ref{alg:decentralizedpowermethod} and our centralized Algorithm~\ref{alg:centralizedpowermethod} in Theorem~\ref{privconvdecalgo} (proof in Appendix~\ref{sec:app_proof5}). 
\begin{definition}
Let $SecAgg({\bm Y_\ell^{(i)}}, \{i | 1{\leq}i{\leq}s\})$ the Secure Aggregation of matrices $\bm Y_\ell^{(i)}$ held  by users indexed by $\{i | 1{\leq}i{\leq}s\})$. It is equivalent to computing $\bm{Y}_\ell = \sum_{i=1}^s \bm Y_\ell^{(i)}$ over secure channels. 
\end{definition}
\begin{algorithm}[ht]
	\begin{algorithmic}[1]
	\caption{Federated private power method}
    \label{alg:decentralizedpowermethod}
	\STATE{\textbf{Input}: distributed matrices $\bm A^{(1)},\cdots,\bm A^{(s)}\in\mathbb R^{n\times n}$, number of iterations $L$, target rank $k$, iteration rank $p\geq k$, 
		private parameters $\epsilon,\delta$.}
	\STATE{\textbf{Output}: approximated eigen-space $\bm X^L\in\mathbb R^{n\times p}$, with orthonormal columns.}
	\STATE{\textbf{Initialization}}: orthonormal $\bm X^0\in\mathbb R^{n\times p}$ by QR decomposition on a random Gaussian matrix $\bm G_0$;
	noise variance parameter $\nu=\epsilon^{-1}\sqrt{\frac{4L\log(1/\delta)}{s}}$;\\
	\FOR{$\ell=1$ to $L$}
	\STATE{
		1. The central node broadcasts $\bm X^{\ell-1}$ to all $s$ computing nodes;\\
		2. Computing node $i$ computes $\bm Y^{(i)}_\ell=\bm A^{(i)}\bm X^{\ell-1}+\bm G^{(i)}_\ell$ with $\bm G^{(i)}_\ell\sim \mathcal{N}(0,\Delta_l^2\nu^2)^{n\times p}$;\\
  
		3. The central node computes with the clients $\bm Y_\ell= SecAgg({\bm Y_\ell^{(i)}}, \{i | 1 \leq i \leq s\})$. \\
  		4. The central node computes QR factorization $\bm Y_\ell=\bm X^\ell\bm R_\ell$.
	}
	\ENDFOR
	\end{algorithmic}
\end{algorithm}
\begin{theorem}[Privacy and utility of Algorithm \ref{alg:decentralizedpowermethod}]
\label{privconvdecalgo}
    The decentralized Privacy-Preserving Power Method (Algorithm~\ref{alg:decentralizedpowermethod}) provides the same privacy guarantees and achieves equivalent utility (in terms of convergence) as its centralized version (Algorithm~\ref{alg:centralizedpowermethod}).
\end{theorem}

\subsection{Communication and computation cost analysis:}
Table~\ref{tbl:overhead_compare} presents the per-iteration communication and computation overhead introduced by both Secure Aggregation~\citep{bell2020secure} and non-secure federated aggregation.
We can see that Secure Aggregation introduces some non-dominant logarithmic terms.

\section{Empirical comparison of the proposed bounds}
We introduced runtime-dependent and runtime-independent convergence bounds for our algorithm. The runtime-dependent bound depends on $\bm{A}$ and the iteratively computed $\bm{X}^l$, whereas the runtime-independent bound (Theorem~\ref{thm:ippmruninde}) only depends on $\bm{A}$. To illustrate the practical impact of our changes, we focus on an application in recommender systems and compare how our algorithms perform in this context. Additionally, we use a statistical approximation to see how the bounds compare at the first step of the algorithm (sketching step), regardless of the application in Appendix~\ref{agnmatrcomp}. 

\subsection{Application to recommender systems}
State-of-the-art recommender systems like GF-CF~\citep{shen2021powerful} and BSPM~\citep{choi2023blurring} utilize singular or eigenvalue decomposition as part of their algorithms. Specifically, they represent the dataset (user-item interactions) using a normalized adjacency matrix $\tilde{\bm{R}}$, from which they compute a normalized item-item matrix $\tilde{\bm{P}}$. To get rid of noise in $\tilde{\bm{R}}$, they compute an ideal low-pass filter based on the top-$p$ eigenvectors of $\tilde{\bm{P}}$ and apply it to $\tilde{\bm{R}}$. Detailed definitions are provided in Appendix~\ref{apprecsys}.

The first four columns of Table~\ref{tbl:stats_datasets} presents key statistics of popular recommendation datasets used in the literature. Depending on the use case, these datasets may be held by a central curator, or distributed among the users of a recommender system. Accordingly, either the centralized PPM (Algorithm~\ref{alg:centralizedpowermethod}) or the decentralized PPM (Algorithm~\ref{alg:decentralizedpowermethod}) can be applied to compute the leading eigenvectors. To demonstrate the practicality of our proposed methods, we focus on a decentralized setting where each user has access to its own ratings and collaborates with other users to compute the desired eigenvectors. $\Delta_l^{prior}=\mu_0(\bm{A})\sqrt{p \cdot  \log({n})}$ denotes the sensitivity bound proposed in \citep{balcan2016improved} while ${\hat{\Delta}_l} = \mu_1(\bm{A})$ uses our multidimensional refinement.

\begin{table}[htb]
    \centering
    \caption{Statistics of datasets}\label{tbl:stats_datasets}
    \begin{tabular}{cccccccc}\toprule
        Dataset     & Users & Items & Interactions & $\mu_0 $& $\mu_1$ & $\Delta_l^{prior} $ (prior) & $ {\hat{\Delta}_l} $ (ours)\\ \midrule
        Amazon-book & 52,643  & 91,599  & 2,984,108      & 0.33 & 0.99 & 1.12 $\times \sqrt{p} $ & 0.99 \\
        MovieLens & 71,567 & 10,677 & 7,972,582      & 0.39 & 1.00 & 1.19 $\times \sqrt{p} $ & 1.00 \\
        EachMovie & 74,425  & 1,649  & 2,216,887      & 0.49 & 1.00 &  1.33 $\times \sqrt{p} $ & 1.00 \\
        Jester & 54,906 & 151 & 1,450,010      & 0.73 & 1.00 & 1.64 $\times \sqrt{p} $ & 1.00 \\
        \bottomrule
    \end{tabular}
\end{table}

\begin{figure}[htb]
\vspace{-5mm}
  \centering
  \subfigure[MovieLens dataset.]{
    \includegraphics[width=0.48\textwidth]{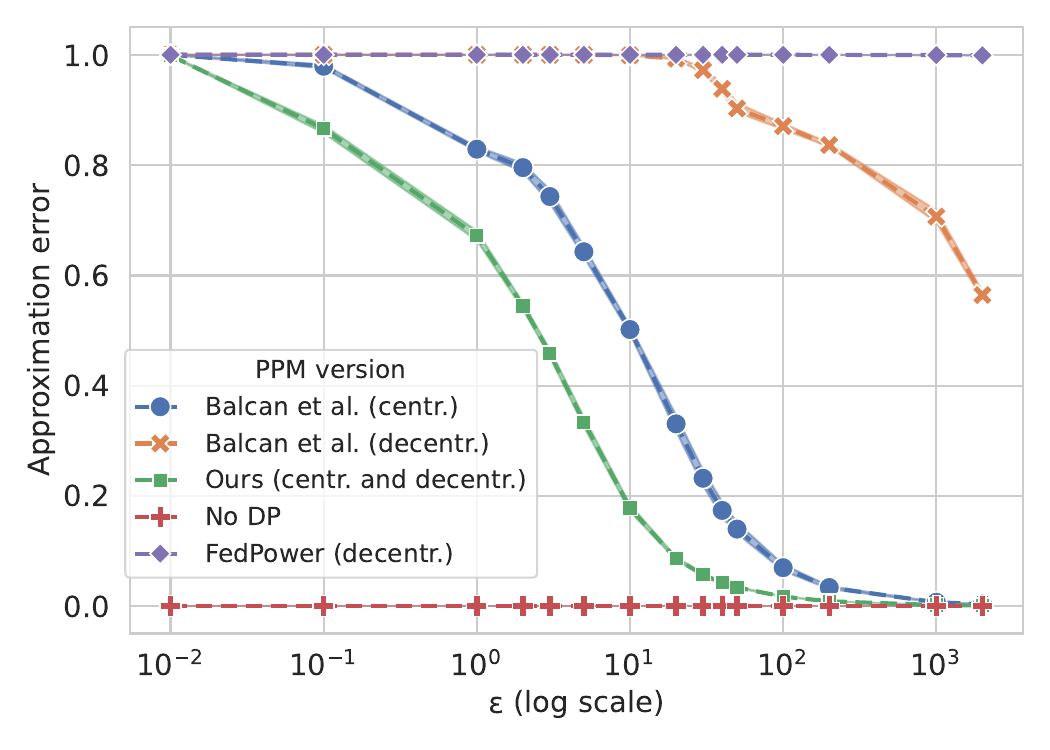}
    \label{fig:frob_movielens}
     \vspace{-3mm}
  }\hfill
  \subfigure[EachMovie dataset.]{
       \vspace{-3mm}
    \includegraphics[width=0.48\textwidth]{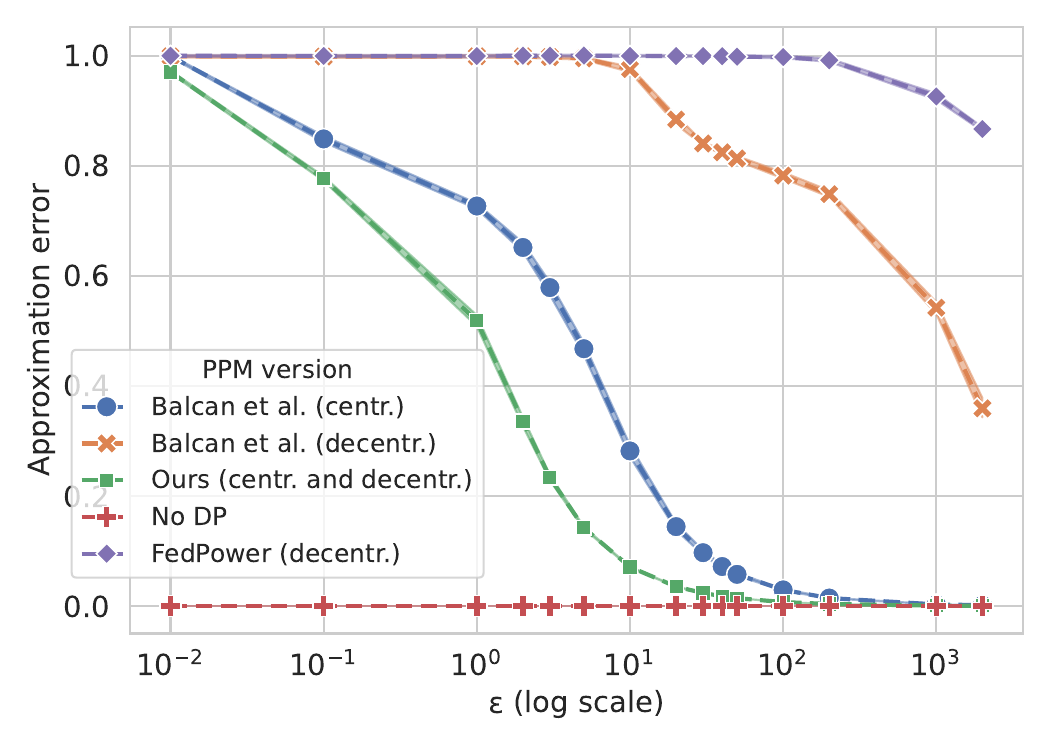}
    \label{fig:frob_eachmovie}
     \vspace{-3mm}
  }
  \vspace{-4mm}
  \centering
  \caption{Comparison of the impact of the Differential Privacy parameter $\epsilon$ on the relative approximation error 
  $
  \frac{\|\tilde{\bm{R}_p} - \bm{R}_p \|_F}{\|\bm{R}_p\|_F},
  $
  where $\tilde{\bm{R}_p}$ represents the approximated matrix and $\bm{R}_p$ the original matrix. Results are means on 10 runs with shaded bands indicating 99\% confidence intervals (computed via bootstrap), shown for the MovieLens and EachMovie datasets, with $p=32$ and $L=3$.}
  \label{fig:frob_comparison}
\end{figure}


\paragraph{Empirical comparison of the runtime-independent bounds.}
\label{empiricalcomparison}
As discussed in Section \ref{runtimeindebounds}, our proposed runtime-independent bound is tighter than those in \citet{hardt2014noisy,balcan2016improved} when $\mu_1(\bm{A}) \leq \mu_0(\bm{A})\sqrt{p \cdot  \log({n})}$, where $\bm{U}$ denotes the eigenvectors of $\bm{A}$. The last two columns from Table~\ref{tbl:stats_datasets} show that this condition holds for popular recommender system datasets. As we can deduce from Table \ref{tbl:stats_datasets}, the proposed bound theoretically allows us to converge to solutions with much smaller $\eta$ values, especially when the desired number of factors $p$ is large.

\paragraph{Practical utility of the proposed algorithms.}
\label{practicalutility}
We saw that the proposed runtime-independent bounds were practically tighter than the previous one for the task of interest. We now demonstrate that Algorithm~\ref{alg:decentralizedpowermethod} can be used to compute the top-$p$ eigenvectors $\bm{U}_p$ of the item-item matrix $\tilde{\bm{P}}$ under ($\epsilon$,$\delta$)-DP (proof in Appendix~\ref{prooflemmaapprecsys}):
\begin{lemma}
Algorithm~\ref{alg:decentralizedpowermethod} with $\Delta_l = \sqrt{2} \max_i \|\bm{X}_{i:}^l \|_2$ can approximate $\bm{U}_p$, the top-$p$ eigenvectors of $\tilde{\bm{P}}$ in a decentralized setting under $(\epsilon, \delta)$-Differential Privacy. 
\end{lemma}

Both GF-CF~\citep{shen2021powerful} and BSPM~\citep{choi2023blurring} use $\bm{U}_p$ to compute the ideal low pass filter and apply it to the interaction matrix $\bm{R}$, yielding $\bm{R}_p$. To illustrate the practicality of our proposed method, we compute approximations of $\bm{U}_p$ (denoted by $\tilde{\bm{U}}_p$) using either our decentralized PPM, the previous PPM versions in ~\citet{balcan2016improved} or FedPower from \citet{guo2024fedpower}. We then use $\tilde{\bm{U}}_p$ to compute an approximation of $\bm{R}_p$ (denoted by $\tilde{\bm{R}_p}$), and compute the relative approximation error $\frac{\|\tilde{\bm{R}_p} - \bm{R}_p \|_F}{\|\bm{R}_p\|_F}$ associated to each version. We provide experimental details in Appendix~\ref{experimentaldetails}.

Figure~\ref{fig:frob_comparison} illustrates the impact of the privacy parameter $\epsilon$ on the approximation error for the EachMovie and MovieLens datasets, with $p=32$ and $L=3$, clearly showing the advantage of our method over prior work. Our decentralized PPM achieves relative approximation error of $\approx 1/10$ for values of $\epsilon \in (5,10)$ for EachMovie and $\epsilon \approx 20$ for MovieLens. In contrast, other decentralized methods \citep{balcan2016improved, guo2024fedpower} require $\epsilon$ to be of the order of at least $10^3$ to achieve comparable errors, which does not seem to provide meaningful privacy protection.\footnote{We refer to \citet{dwork2019differential} for more practical details on how to set $\epsilon$ and provide additional comparisons and details in Appendix~\ref{approximation_errors}.} \citep{guo2024fedpower} yields the worst relative approximation errors. We hypothesize that this is because it uses worst case, non-adaptive sensitivity bounds for DP, as opposed to our proposed method and those of \citet{balcan2016improved}. For a fixed approximation error, both our centralized and decentralized methods yield $\epsilon$ roughly four times smaller than required by the centralized method of \citet{balcan2016improved}, demonstrating that our propositions significantly strengthen privacy guarantees.

\section{Conclusion }
\label{sec:conclusion}
We presented differentially private versions of the centralized and decentralized randomized power method that addresses privacy concerns in large-scale spectral analysis and recommendation systems. We introduced a new sensitivity bound, which we show theoretically and empirically to improve the accuracy of the method while ensuring privacy guarantees. By employing Secure Aggregation in a decentralized setting, we can reduce the noise introduced for privacy, maintaining the efficiency and privacy of the centralized version but adapting it for distributed environments. Our methods could enable organizations (from healthcare networks to social platforms) to extract useful structure from data without compromising individual records.

{\bf Limitations.} Our privacy guarantees require no dropout from participants to ensure distributed DP, which may not always be realistic. It would also be interesting to analyze a private and decentralized version of the accelerated version of the randomized power method, which could potentially converge faster.



\clearpage 
\bibliography{bibliography}
\bibliographystyle{ACM-Reference-Format}
\appendix
\onecolumn
\section{Proofs of Results}
\label{sec:app_proof}
\subsection{Proof of Theorem~\ref{improvedsensitivitybound}} 
\label{sec:app_proof1}

\begin{proof}

We denote by $\bm{A}'$ a matrix adjacent to $\bm{A}$ using (\ref{proposedupdatemodel}):
\begin{equation}
    \bm{A}' =\bm{A} + \bm{C}, 
\end{equation}
with $\bm{C}$ a symmetric matrix representing the update, subject to $\sqrt{\sum_i \|\bm{C}_{i:}\|_1^2} \leq 1$.

Then 
\begin{equation}
\label{decompositionindep}
\begin{split}
    \Delta_l &= \|\bm{A} \bm{X}^l - \bm{A}'\bm{X}^l \|_F \\
                   &= \|\bm{A}\bm{X}^l - (\bm{A} +  \bm{C}) \bm{X}^l \|_F \\
                   &= \|\bm{C} \bm{X}^l \|_F \\
                   &= \sqrt{\sum_i \|\bm{C}_{i:} \bm{X}^l \|_F^2}.\\ 
\end{split}
\end{equation} and
\begin{equation}
\label{linebound}
\begin{split}
    \|\bm{C}_{i:} \bm{X}^l \|_F &= \| \sum_j \bm{C}_{ij} \cdot \bm{e}_j^\top \cdot \bm{X}^l \|_F\\
                   &= \| \sum_j \bm{C}_{ij} \cdot \bm{X}^l_{j:} \|_F\\
                   & \leq \sum_j \|  \bm{C}_{ij} \cdot \bm{X}^l_{j:} \|_F\\
                   & \leq \sum_j |\bm{C}_{ij}| \cdot  \| \bm{X}^l_{j:} \|_F\\
                   & \leq \max_i \|\bm{X}_{i:}^l\|_F \cdot  \sum_j |\bm{C}_{ij}|. 
\end{split}
\end{equation}
By injecting (\ref{decompositionindep}) into (\ref{linebound}), we have:
\begin{equation}
\label{finalbound}
\begin{split}
    \Delta_l &= \sqrt{\sum_i \|\bm{C}_{i:} \bm{X}^l \|_F^2}\\
            & \leq \sqrt{\sum_i( \max_i \|\bm{X}_{i:}^l\|_F \cdot  \sum_j |\bm{C}_{ij}|)^2} \\
            & \leq \max_i \|\bm{X}_{i:}^l\|_F \sqrt{\sum_i(  \cdot  \sum_j |\bm{C}_{ij}|)^2} \\
            & \leq \max_i \|\bm{X}_{i:}^l\|_F \sqrt{\sum_i \|\bm{C}_{i:}\|_1^2} \\
            & \leq \max_i \|\bm{X}_{i:}^l\|_F \triangleq \hat{\Delta}_l. \\
\end{split}
\end{equation}

Although the proposed update model (\ref{proposedupdatemodel}) is more general than when $\bm{A}'$ is defined using (\ref{pastupdatemodel}),  the proposed bound $\hat{\Delta}_l$ is also generally tighter than the bound proposed in \cite{hardt2014noisy, balcan2016improved}.

\end{proof}
\subsection{Privacy proof}
\label{note:note_proofs}
\subsubsection{Note on related privacy proofs:} 
Several Differential Privacy (DP) proofs for the private randomized power method (PPM) have been developed in prior works, for instance those by \citet{hardt2012beating, hardt2013beyond, hardt2014noisy,balcan2016improved}.
These papers mostly rely on the same proofs to establish that the PPM satisfies ($\epsilon$, $\delta$)-DP. Specifically, \citet{balcan2016improved} references the privacy proof from \citet{hardt2014noisy}, which builds upon the privacy proof from \citet{hardt2013beyond}. The proof therein is also closely related to the one of \citet{hardt2012beating}.

 However, both \citet{hardt2013beyond} (Theorem 2.4) and \citet{hardt2012beating} (Theorem 2.4) contain errors in their proposed composition rules, where a comparison sign is mistakenly flipped. This error could potentially cause the privacy parameter $\epsilon$ in the proposed mechanism to be arbitrarily small, providing no privacy guarantee at all. The original composition rule is presented in Theorem III.3 of \citet{dwork2010boosting}. Moreover, Lemma 3.4 of \citet{hardt2013beyond} misuses Theorem 2.4. Indeed, they claim that their algorithm satisfies $(\epsilon', \delta)$-DP at each iteration. By their Theorem 2.4, then the algorithm overall should satisfy $(\epsilon', k \delta + \delta')$-DP where $\delta' > 0$, and not $(\epsilon', \delta)$-DP as claimed. 

 \subsubsection{Zero-Concentrated Differential Privacy (zCDP):}
A randomized algorithm $ \mathcal{M} $ is said to satisfy $ \rho $-zero-Concentrated Differential Privacy (zCDP) if for all neighboring datasets $ D_1 $ and $ D_2 $, and for all $\alpha \in (1, \infty)$, the following holds:
\begin{align}
D_{\alpha}(\mathcal{M}(D_1) \| \mathcal{M}(D_2)) \leq \rho \alpha\,,
\end{align}
where $ D_{\alpha} $ is the Rényi divergence of order $ \alpha  $ and  $ \rho $ is a positive parameter controlling the trade-off between privacy and accuracy (smaller values of $ \rho $ imply stronger privacy guarantees). 

The following lemma, introduced by~\citet{bun2016concentrated}, specifies how the addition of Gaussian noise can be used to construct a randomized algorithm that satisfies zCDP. 
\begin{lemma}[Gaussian Mechanism (Proposition 1.6 \citep{bun2016concentrated})]
\label{gaussian_mechanism}
    Let $ f : X^n \to \mathbb{R} $ be a sensitivity-$ \Delta $ function. Consider the mechanism (randomized algorithm) $ M : X^n \to \mathbb{R} $, defined as $M(x) = f(x) + Z_x$ where for each input $x$, $Z_x$ is independently drawn from $ \mathcal{N}(0, \sigma^2) $. Then $ M $ satisfies $ \left(\frac{\Delta^2}{2\sigma^2}\right) $-zCDP.
\end{lemma}

The next lemma, which is a generalized version of a result presented by~\citet{bun2016concentrated}, explains how a randomized algorithm, constructed by recursively composing a sequence of zCDP-satisying randomized algorithms, also satisfies zCDP. 
\begin{lemma}[Adaptive composition (Generalization from Lemma 2.3 of \citep{bun2016concentrated})]
\label{adaptive_composition}
Let $ M_1 : X^n \to Y_1 $, $ M_2 : X^n \times Y_1 \to Y_2 $, \dots, $ M_L : X^n \times Y_1 \times \dots \times Y_{L-1} \to Y_L $ be randomized algorithms. Suppose $ M_i $ satisfies $ \rho_i $-zCDP as a function of its first argument for each $ i = 1, 2, \dots, L $. Let $ M^{\prime\prime} : X^n \to Y_L $, constructed recursively by:
\begin{align}
M^{\prime\prime}(x) = M_L(x, M_{L-1}(x, \dots, M_2(x, M_1(x)) \dots)).
\end{align}
Then $ M^{\prime\prime} $ satisfies $ (\sum_{i=1}^L \rho_i) $-zCDP.
\end{lemma}

\subsubsection{Proof of Theorem~\ref{thm:privacyofppm}}
\label{sec:app_proof2}

\begin{proof}
We can see lines 4-5 of Algorithm \ref{alg:centralizedpowermethod} as a sequential composition ($M$) of $L$ Gaussian Mechanisms. By Lemma~\ref{gaussian_mechanism}, each mechanism $M_i$ satisfies ($\frac{\Delta_i^2}{2\sigma_i^2}$)-zCDP where $\Delta_i$ is the $\ell_2$-sensitivity of the function associated to mechanism $\Delta_i$ and $\sigma_i^2$ is the variance of the noise added with the Gaussian Mechanism. By Lemma~\ref{adaptive_composition}, the composition of mechanisms $M = (M_1, ... , M_i, ... , M_L)$ satisfies ($\sum_{i=1}^L \frac{\Delta_i^2}{2\sigma_i^2}$)-zCDP. Let $\rho \triangleq \sum_{i=1}^L \frac{\Delta_i^2}{2\sigma_i^2}$. By design of our algorithm, we have:
\begin{align}
    \rho &= \sum_{i=1}^L \frac{\Delta_i^2}{2\sigma_i^2}\\
    & \leq \sum_{i=1}^L \frac{(\hat{\Delta}_l)^2}{2\sigma_i^2}\\
    &= \sum_{i=1}^L \frac{1}{2 \sigma^2} \\
    &= \sum_{i=1}^L \frac{\epsilon^{2}}{ {8L\log(1/\delta)}} \\
    &=  \frac{\epsilon^{2}}{ {8\log(1/\delta)}}. 
\end{align}
By Proposition 1.3 of \citet{bun2016concentrated}, if M provides $\rho$-zCDP, then M is ($\rho + 2 \sqrt{\rho \log(1/\delta)}, \delta)$-DP, $ \forall \delta>0$. Let $\epsilon' \triangleq \rho + 2 \sqrt{\rho \log(1/\delta)}$. Then:
\begin{align}
    \epsilon' & = \rho + 2 \sqrt{\rho \log(1/\delta)} \\
    & \leq \frac{\epsilon^{2}}{ {8\log(1/\delta)}} + 2 \sqrt{\frac{\epsilon^{2}}{ {8\log(1/\delta)}}.  \log(1/\delta)} \\
    &\leq \frac{\epsilon^{2}}{ {8\log(1/\delta)}} + \frac{\epsilon}{2}. \\
\end{align}
To satisfy ($\epsilon$, $\delta$)-DP, we need:
\begin{align}
    \epsilon \geq \epsilon' 
    &\impliedby \epsilon \geq \frac{\epsilon^{2}}{  {8\log(1/\delta)}} + \frac{\epsilon}{2} \\
    &\iff \frac{\epsilon}{2} \geq \frac{\epsilon^{2}}{ {8\log(1/\delta)}}  \\    
    &\iff  \epsilon  \leq  4\log(1/\delta) \\   
    &\iff \delta \leq \exp{(-\frac{\epsilon}{4})},   
\end{align}
which is a reasonable assumption, as in practice $\epsilon = O(1)$ and $\delta \ll \frac{1}{d}$, where $d$ is the number of records to protect. In our case, $d=n^2$ because we run the privacy-preserving power method on $\bm{A}\in \mathbb{R}^{n \times n}$.
\end{proof}

\subsection{Proof of Theorem~\ref{thm:proposed_upper_bound}}
\label{proofupperboundindep}
The following theorem from \citet{balcan2016improved} will be useful in our proof:
\begin{theorem}[Bound for the noisy power method (NPM) \citep{balcan2016improved}]\label{thm:balcanthm}
	Let $k\leq q \leq p$ be positive integers. Let $\bm U_q\in\mathbb R^{d\times q}$ be the 
	top-$q$ eigenvectors of a positive semi-definite matrix $\bm A$ and let $\lambda_1\geq\cdots\geq\lambda_d\geq 0$ denote
	its eigenvalues and fix $\eta = O\left(\frac{\lambda_q}{\lambda_k}\cdot\min\left\{\frac{1}{\log\left(\frac{\lambda_k}{\lambda_q}\right)},\frac{1}{\log\left(\tau d\right)}\right\}\right)$.
	If at every iteration $l$ of the NPM $\bm G_\ell$ satisfies Assumption~\ref{assum_check}, then after
	$$
	L = \Theta\left(\frac{\lambda_k}{\lambda_k-\lambda_{q+1}}\log\left(\frac{\tau d}{\eta }\right)\right).
	$$ iterations, with probability at least $1-\tau^{-\Omega(p+1-q)}-e^{-\Omega(d)}$, we have:
	\begin{align}
	\|(\bm I-\bm X^L{\bm X^L}^\top)\bm U_k\|_2\leq\eta \quad \text{and} \quad   \| (\bm{I} - \bm{X}^L ({\bm{X}^L})^\top  ) \bm{A}\|_2^2 \leq \lambda_{k+1}^2 + \eta^2  \lambda_k^2.
	\end{align}
\end{theorem}

We now provide the proof of Theorem~\ref{thm:proposed_upper_bound}:
\label{sec:app_proof3}
\begin{proof}
    According to \cite{hardt2014noisy}, if $\bm{G}_l \sim \mathcal{N}(0, \sigma_l)^{d \times p}$, then with probability 99/100 we have:
\begin{equation}
\begin{split}
        &\max_{l=1}^L \| \bm{G}_l \| \leq \sigma_l \cdot \sqrt{d \cdot \log(L)}\,, \\
        &\max_{l=1}^L \| \bm{U}^{\top} \bm{G}_l \| \leq \sigma_l \cdot \sqrt{p \cdot \log(L)}\,.
\end{split}
\end{equation}
We can therefore satisfy the noise conditions of Theorem~\ref{thm:balcanthm} with probability 99/100 if we choose $\eta = \frac{\sigma_l \cdot \sqrt{d \cdot \log(L)}}{\lambda_k - \lambda_{q+1}}$, giving us:
\begin{equation}
    \|(\bm{I} - \bm{X}^L {\bm{X}^L}^{\top}) \cdot \bm{U}_k \| \leq \frac{\sigma_l \cdot \sqrt{d \cdot \log(L)}}{\lambda_k - \lambda_{q+1}}\,,
\end{equation}
which leads us to the statement of the theorem by injecting $\sigma_l = \hat{\Delta}_l \cdot  \epsilon^{-1}\sqrt{4L\log(1/\delta)}$. 
\end{proof}

\subsection{Proof of Theorem~\ref{thm:ippmruninde}}
\label{sec:app_proof5}

We provide here the proof for Theorem~\ref{thm:ippmruninde}:
\ppmruntimeinde*
\begin{proof}
We showed in Equation that $\Delta_l \leq \max_i \|\bm{X}_{i:}^l\|_F$. This quantity depends on values computed during the execution of the algorithm. We now show that we can bound $\max_i \|\bm{X}_{i:}^l\|_F$ with a runtime-independent bound.

Without loss of generality and to simplify notation, we use $\bm{X}$ to denote any matrix $\bm{X}_l$ computed during the execution of the Private Power Method. Let $\bm{X}_{:c} = x_c$ denote a column of $\bm{X}$.

As $\bm{A}$ is an Hermitian matrix, by the spectral theorem, we have $\bm{A} = \bm{U} \bm{D} \bm{U}^\top$, where $\bm{U}$ is unitary (with orthonormal columns) and $\bm{D}$ is diagonal.

As the columns of $\bm{U}$ form a complete basis for $\mathbb{R}^n$, we can write any column $x_c$ of $\bm{X}$ as $\sum_{i=1}^n \alpha_c^i u_i$, where $u_i$ denotes the $i$-th eigenvector of $\bm{A}$, and $ \alpha_c^i$ is a scalar. 

We can then write:
\begin{align}
    \langle x_c,x_e \rangle &= \langle\sum_{i=1}^n \alpha_c^i u_i,\sum_{j=1}^n \alpha_e^j u_j\rangle  \\ 
    &= \sum_{i=1}^n \sum_{j=1}^n  \alpha_c^i \alpha_e^j \langle u_i, u_j\rangle \\ 
    &= \sum_{i=1}^n \alpha_c^i \alpha_e^i.  && \text{\footnotesize(Orth. columns of $\bm{U}$)} 
\end{align}   
$\bm{X}$ is the ``$Q$`` matrix constructed from a Gram-Schmidt QR decomposition, it has therefore orthonormal columns by definition. Therefore, if $c = e$, then $\langle x_c,x_e\rangle = \sum_{i=1}^n {(\alpha_c^i)}^2 =1$. Otherwise, we have $\langle x_c,x_e \rangle = \sum_{i=1}^n \alpha_c^i \alpha_e^i=0$. 

The key is then to notice that we can define a matrix $\bm{B} \in \mathbb{R}^{n \times p}$ with orthonormal columns such that $\bm{X} = \bm{U}\bm{B}$ and $\bm{B}_{jc} = \alpha_c^j$. 

We recall that any matrices $\bm{H}$ and $\bm{J}$, we have $\|\bm{H}\bm{J}\|_F \leq \|\bm{H}\|_2\|\bm{J}\|_F$.

Then, we can bound the norm of any row of $\bm{X}_{i:}$ as:
    \begin{align}
        \|\bm{X}_{i:} \|_F &= \|\bm{X}_{i:}^\top \|_F \,,\\
        & = \| \bm{B}^\top \bm{U}_{i:}^\top\|_F \,,\\
        & \leq \| \bm{B}^\top\|_2 \|\bm{U}_{i:}^\top\|_F.
    \end{align}
$\bm{B}$ has orthonormal columns by construction, therefore 
\begin{align}
    \|\bm{B}\|_2 &= \sqrt{\|\bm{B}^\top \bm{B}\|_2} \\
    &= \sqrt{\| \bm{I}\|_2} \\
    &= 1.
\end{align}

We then have:
\begin{align}
    \|\bm{X}_{i:} \|_F &\leq \| \bm{B}^\top\|_2 \|\bm{U}_{i:}^\top\|_F \\
      &\leq \|\bm{U}_{i:}^\top\|_F \\
     & \leq \|\bm{U}_{i:}\|_F.
\end{align}

Additionally as $\bm{X}_{i:} \in \mathbb{R}^{1 \times n}$, $\max_i \|\bm{X}_{i:}\|_2 = \max_i \|\bm{X}_{i:}\|_F   \leq \max_i (\|\bm{U}_{i:}\|_F) = \mu_1(\bm{A})$.

\textbf{Note:} By Section 2.4 from \citet{woodruff2014sketching}, for any row $\bm{v}$ of a matrix with orthonormal columns $\bm{Z}$, $\|v\|_2 \leq 1$. 

As $\bm{U}$ has orthonormal columns by construction, $\max_i \|\bm{U}_{i:} \|_2 \leq 1$.

We can therefore bound $\max_i\|\bm{X}_{i:} \|_2$ as:
\begin{align}
     \Delta_l  \leq \max_i\|\bm{X}_{i:} \|_2 \leq \mu_1(\bm{A})  \leq 1.
\end{align}

Injecting this bound in Theorem~\ref{thm:proposed_upper_bound} leads us to Theorem~\ref{thm:ippmruninde}, giving us a runtime-independent bound.
\end{proof}

\subsection{Proof of Theorem~\ref{privconvdecalgo}}
\begin{proof}
    It is straightforward to see that if steps 2 and 3 of Algorithm \ref{alg:decentralizedpowermethod} are equivalent to step 1 of Algorithm \ref{alg:centralizedpowermethod}, then the two algorithms are equivalent. 
Recall that $Y^{(i)}_\ell=\bm A^{(i)}\bm X^{\ell-1}+\bm G^{(i)}_\ell$ and $\bm G^{(i)}_\ell\sim \mathcal{N}(0,\Delta_l^2\nu^2)^{n\times p}$. Then steps 2 and 3 of Algorithm \ref{alg:decentralizedpowermethod} correspond to:
\begin{equation}
\begin{split}
    		\bm Y_\ell &= SecAgg({\bm Y_\ell^{(i)}}, \{i | 1 \leq i \leq s\})\\
            &= \sum_{i=1}^s \bm Y_\ell^{(i)} \\
            &= \sum_{i=1}^s (\bm A^{(i)}\bm X_{\ell-1}+\bm G^{(i)}_\ell ) \\
            &= \bm A \bm X^{\ell-1}+\sum_{i=1}^s\bm G^{(i)}_\ell \\
            &= \bm A \bm X^{\ell-1}+\bm G_\ell,
\end{split}
\end{equation}
where $\bm G_\ell\sim \mathcal{N}(0,\Delta_l^2 \cdot  (s \nu^2))^{n\times p}$, and we have $s \nu^2 = \sigma^2$ by definition, completing the equivalence proof.
\end{proof}

\section{Application to recommender systems:}
\label{apprecsys}
Let $s$ be the number of users in our system and $n$ the number of items.
Let $\bm{R} \in \mathbb{R}^{s \times n}$ be the user-item interaction matrix, such that $\bm{R}_{ui} = 1$ only if user $u$ has interacted with item $i$, and $\bm{R}_{ui} = 0$ else.
Let $\bm{U} = \text{Diag}(\bm{R} \cdot \bm{1}_{|\mathcal{I}|})$ be the user degrees matrix, and $\bm{I} = \text{Diag}(\bm{1}_{|\mathcal{I}|}^\top \cdot \bm{R})$ the item degrees matrix.
In \citet{shen2021powerful,choi2023blurring}, the normalized interaction matrix is defined as: 
\begin{align*}
    \bm{R}' = \bm{U}^{-\frac{1}{2}} \bm{R} \bm{I}^{-\frac{1}{2}},
\end{align*}
and the item-item normalized adjacency matrix as:
\begin{equation}
\begin{split}
\bm{P}' & =\tilde{\bm{R}}^\top\tilde{\bm{R}} \\
& = ( \bm{U}^{-\frac{1}{2}} \bm{R})^\top( \bm{U}^{-\frac{1}{2}} \bm{R} ) \\
&=  (\bm{U}^{-\frac{1}{2}}\bm{R} \bm{I}^{-\frac{1}{2}})^\top(\bm{U}^{-\frac{1}{2}} \bm{R} \bm{I}^{-\frac{1}{2}} ). \\
\end{split}
\end{equation}

To simplify our analysis, we consider $\bm{I}$ public and do not use item-wise normalization in the computation of the ideal low pass filter, leaving it for future work. We therefore define $\tilde{\bm{R}} = \bm{U}^{-\frac{1}{2}} \bm{R} $ and $\tilde{\bm{P} } =  (\bm{U}^{-\frac{1}{2}}\bm{R} )^\top(\bm{U}^{-\frac{1}{2}} \bm{R} )$. 

\begin{lemma}
We can use Algorithm~\ref{alg:decentralizedpowermethod} with $\Delta_l = \sqrt{2} \max_i \|\bm{X}_{i:}^l \|_2$ to compute the top-$p$ eigenvectors of $\tilde{\bm{P}}$ in a decentralized setting with under a $(\epsilon, \delta)$-Differential Privacy guarantee.  
\end{lemma}
\label{prooflemmaapprecsys}
\begin{proof}
    Let $\tilde{\bm{P}}_{ij}$ denote the element of $\tilde{\bm{P}}$ at row $i$ and column $j$ and let $d_{user}(u) =  \sum_{i=0}^{n-1} r_{ui}$. We can write $\tilde{\bm{P}}_{ij}$ as:
\begin{equation}
\begin{split}
\tilde{\bm{P}}_{ij} &= ((\bm{U}^{-\frac{1}{2}}\bm{R} )^\top(\bm{U}^{-\frac{1}{2}} \bm{R} ))_{ij} \\
& = (\bm{U}^{-\frac{1}{2}}\bm{R})^\top)_{i,*}( \bm{U}^{-\frac{1}{2}}\bm{R})_{*,j} \\
&=  ((\bm{U}^{-\frac{1}{2}}\bm{R})_{*,i})^\top( \bm{U}^{-\frac{1}{2}}\bm{R})_{*,j} \\
&=  \sum_{u=0}^{s-1} \frac{1}{\sqrt{d_{user}(u) }}r_{ui}\frac{1}{\sqrt{d_{user}(u) }}r_{uj} \\
&=  \sum_{u=0}^{s-1} \frac{1}{d_{user}(u)} \cdot r_{ui} \cdot r_{uj}. \\
\end{split}
\end{equation}

By noticing that $(\bm{R} ^\top\bm{R} )_{ij} =  \sum_{u} r_{ui} \cdot r_{uj}$, we can deduce that $\tilde{\bm{P}} = \sum_{u} \frac{1}{d_{user}(u)} \bm{R}_u ^\top\bm{R}_u$. Therefore $\tilde{\bm{P}}$ is partitioned among $s$ users as described in Section~\ref{decentrversion}.

Sensitivity: We protect the user at the item-level and use the deletion model of differential privacy to compute the sensitivity of the PPM used with the item-item normalized adjacency matrix ($\bm{A} = \tilde{\bm{P}}$). Therefore we have:
\begin{equation}
    \bm{A}_{ij} =  \sum_{u=0}^{s-1} \frac{1}{d_{user}(u)} \cdot r_{ui} \cdot r_{uj},
\end{equation}
and
\begin{equation}
        \bm{A}'_{ij} =  \sum_{u=0}^{s-1} \frac{1}{d_{user}(u)-1} \cdot r_{ui}' \cdot r_{uj}',
\end{equation}
where $r_{ui}' = r_{ui}$ except for one user-item interaction, \emph{i.e.}, $r_{vk} =1$ but $r_{vk}' = 0$.
Let $\bm{C} = \bm{A} - \bm{A}'$. Let $\mathcal{N}(v)$ be the set of items which user $v$ liked before deletion.
We have:
\begin{align}
    \sum_i \|\bm{C}_{i:}\|_1^2 &= \sum_i (\sum_j |\sum_{u=0}^{s-1} \frac{1}{d_{user}(u)} \cdot r_{ui} \cdot r_{uj} - \sum_{u=0}^{s-1} \frac{1}{d_{user}(u)-1} \cdot r_{ui}' \cdot r_{uj}'|)^2 \\
    &= \sum_i (\sum_j |\frac{1}{d_{user}(v)} \cdot r_{vi} \cdot r_{vj} -  \frac{1}{d_{user}(v)-1} \cdot r_{vi}' \cdot r_{vj}'|)^2 \\
    &= \sum_{i \in \mathcal{N}(v)} (\sum_{j \in \mathcal{N}(v)} |\frac{1}{d_{user}(v)} \cdot r_{vi} \cdot r_{vj} -  \frac{1}{d_{user}(v)-1} \cdot r_{vi}' \cdot r_{vj}'|)^2 \\
    &= \sum_{i \in \{\mathcal{N}(v)\backslash k\}} (\sum_{j \in \mathcal{N}(v)} |\frac{1}{d_{user}(v)} \cdot r_{vi} \cdot r_{vj} -  \frac{1}{d_{user}(v)-1} \cdot r_{vi}' \cdot r_{vj}'|)^2 \\ & \quad \quad + (\sum_{j \in \mathcal{N}(v)} |\frac{1}{d_{user}(v)} \cdot r_{vk} \cdot r_{vj} -  \frac{1}{d_{user}(v)-1} \cdot r_{vk}' \cdot r_{vj}'|)^2.  
\end{align}
We have:
\begin{align}
    &\sum_{i \in \{\mathcal{N}(v)\backslash k\}} (\sum_{j \in \mathcal{N}(v)} |\frac{1}{d_{user}(v)} \cdot r_{vi} \cdot r_{vj} -  \frac{1}{d_{user}(v)-1} \cdot r_{vi}' \cdot r_{vj}'|)^2 \\
    &=\sum_{i \in \{\mathcal{N}(v)\backslash k\}} (\sum_{j \in \{\mathcal{N}(v)\backslash k\}} |\frac{1}{d_{user}(v)} \cdot r_{vi} \cdot r_{vj} -  \frac{1}{d_{user}(v)-1} \cdot r_{vi}' \cdot r_{vj}'| + \frac{1}{d_{user}(v)} \cdot r_{vi} \cdot r_{vk}  )^2 \\
    &=(d_{user}(v) -1)\{(d_{user}(v) -1) |\frac{1}{d_{user}(v)} -  \frac{1}{d_{user}(v)-1} | + \frac{1}{d_{user}(v)}  \}^2 \\
    &=(d_{user}(v) -1)\{|\frac{1}{d_{user}(v)} | + \frac{1}{d_{user}(v)}  \}^2 \\
    &=\frac{2 (d_{user}(v) -1)}{d_{user}(v)^2},
\end{align}
and 
\begin{align}
    &(\sum_{j \in \mathcal{N}(v)} |\frac{1}{d_{user}(v)} \cdot r_{vk} \cdot r_{vj} -  \frac{1}{d_{user}(v)-1} \cdot r_{vk}' \cdot r_{vj}'|)^2 \\
    &= (\sum_{j \in \mathcal{N}(v)} |\frac{1}{d_{user}(v)}\cdot r_{vj} |)^2 \\
    &= (d_{user}(v)|\frac{1}{d_{user}(v)} |)^2 \\
    &= 1. 
\end{align}
By noticing that $\frac{2 (d_{user}(v) -1)}{d_{user}(v)^2} \leq 1$, we can deduce that: \begin{align}
    \sum_i \|\bm{C}_{i:}\|_1^2 &= \sum_{i \in \{\mathcal{N}(v)\backslash k\}} (\sum_{j \in \mathcal{N}(v)} |\frac{1}{d_{user}(v)} \cdot r_{vi} \cdot r_{vj} -  \frac{1}{d_{user}(v)-1} \cdot r_{vi}' \cdot r_{vj}'|)^2 \\ & \quad \quad + (\sum_{j \in \mathcal{N}(v)} |\frac{1}{d_{user}(v)} \cdot r_{vk} \cdot r_{vj} -  \frac{1}{d_{user}(v)-1} \cdot r_{vk}' \cdot r_{vj}'|)^2  \\
    &= \frac{2 (d_{user}(v) -1)}{d_{user}(v)^2} +1\\
    & \leq 2 \\
    & \implies \sqrt{\sum_i \|\bm{C}_{i:}\|_1^2} \leq \sqrt{2}.
\end{align}
By Equation~\ref{finalbound}, 
\begin{align}
     &\Delta_l  \leq \max_i \|\bm{X}_{i:}^l\|_F \sqrt{\sum_i \|\bm{C}_{i:}\|_1^2} \\
     & \implies \Delta_l  \leq \max_i \|\bm{X}_{i:}^l\|_F \sqrt{2}.
\end{align}
\end{proof}

\subsection{Approximation errors comparisons:}

\label{approximation_errors}
As explained in Section~\ref{practicalutility}, GF-CF and BSPM (with its parameter $T_b =1$) use $\bm{U}_p$ to compute the ideal low pass filter and filter the interaction matrix $\bm{R}$, yielding $\bm{R}_p$. Indeed, we have:
\begin{align}
    \bm{R}_p = \bm{R} \cdot \bm{I}^{-\frac{1}{2}} \bm{U}_p \bm{U}_p^\top \bm{I}^{\frac{1}{2}}.
\end{align}

To study the impact of Differential Privacy on our system, we compute approximations of $\bm{U}_p$ (denoted by $\tilde{\bm{U}}_p$) using our proposed decentralized PPM, the PPM versions of \citet{balcan2016improved} or FedPower from \citet{guo2024fedpower}. We have:
\begin{align}
    \tilde{\bm{R}}_p = \bm{R} \cdot \bm{I}^{-\frac{1}{2}} \tilde{\bm{U}}_p \tilde{\bm{U}}_p^\top \bm{I}^{\frac{1}{2}}.
\end{align}
We then define the relative approximation error caused by the use of differential privacy as $\frac{\|\tilde{\bm{R}_p} - \bm{R}_p \|_F}{\|\bm{R}_p\|_F}$.

We use $L=3$ to run the Power Method as it is the default hyper-parameter choice from \citet{shen2021powerful, choi2023blurring}. We use $p=32$ to have acceptable approximation error for reasonable values of $\epsilon$ (5-10). We use the synchronous version of FedPower \citep{guo2024fedpower} to simplify the comparison, \emph{i.e.}, we set their parameter $\mathcal{I}_T = L$. We note that FedPower could be improved by also using Secure Aggregation and therefore reducing the noise necessary for DP. It might also benefit from our sensitivity analysis in the synchronous setting.

We showed the impact of the Differential Privacy parameter $\epsilon$ on the approximation errors for the MovieLens and EachMovie datasets in Figures~\ref{fig:frob_movielens} and~\ref{fig:frob_eachmovie}.The trends for the approximation errors on the Jester dataset are quite similar, as shown in Figure~\ref{fig:frob_jester}. We however note that all methods perform better on this dataset. We hypothesize that this is because the Jester dataset is much more dense compared to the EachMovie and MovieLens datasets, hence the ratio of magnitude of the noise added due to DP compared to the magnitude of the elements of the item-item matrix is smaller on this dataset.

\begin{figure}[htb]
    \centering
    \includegraphics[width=0.9\linewidth]{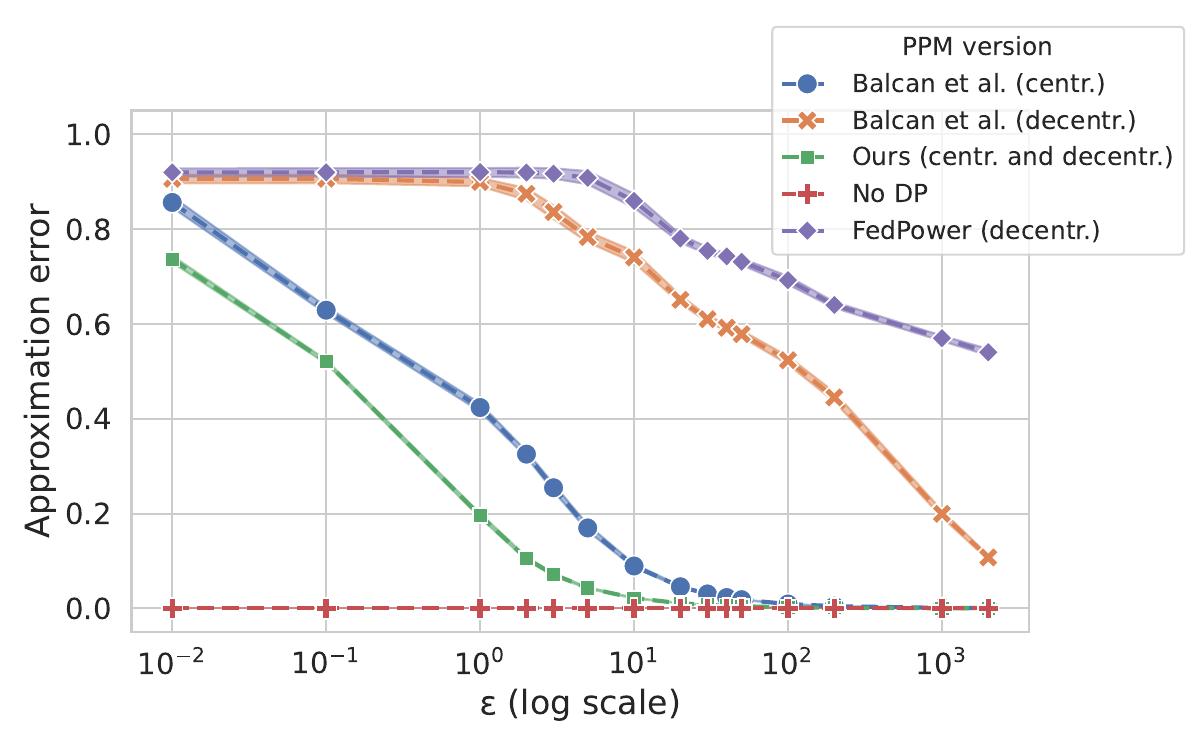}
    \caption{Impact of the Differential Privacy parameter $\epsilon$ on the relative approximation error $\frac{\|\tilde{\bm{R}_p} - \bm{R}_p \|_F}{\|\bm{R}_p\|_F}$ associated to multiple PPM versions, computed on 10 runs with 99\% confidence intervals (computed via bootstrap), with $p=32$ and $L=3$ on the Jester dataset.}
    \label{fig:frob_jester}
\end{figure}
\subsection{Experimental Details}
\label{experimentaldetails}
\paragraph{Datasets.}
We benchmark our method on four standard recommendation datasets: MovieLens-10M, EachMovie, Jester, and Amazon-book. For Amazon-book, we use the publicly available 80\%/20\% train/test splits; for EachMovie, MovieLens-10M and Jester we use the same train/test proportions and therefore apply a 80\%/20\% split at the user level (each user retains at least one test interaction).  Table~\ref{tbl:dataset_stats} summarizes the post-processing statistics.

\begin{table}[ht]
	\centering
	\caption{Dataset statistics after preprocessing: number of users, items, interactions.}
	\label{tbl:dataset_stats}
	\begin{tabular}{lrrr}
		\toprule
		Dataset       & Users  & Items  & Interactions \\
		\midrule
		MovieLens-10M & 71{,}567 & 10{,}677 & 7{,}972{,}582 \\
		EachMovie     & 74{,}425 &  1{,}649 & 2{,}216{,}887 \\
		Jester        & 54{,}906 &    151   & 1{,}450{,}010 \\
		Amazon-book   & 52{,}643 &  91{,}599 & 2{,}984{,}108 \\
		\bottomrule
	\end{tabular}
\end{table}

\paragraph{Hyperparameters.}
We fix the number of power iterations to \(L=3\) to correspond to PyTorch’s \texttt{svd\_lowrank} approximate basis subroutine (Algorithm 4.3 in \citep{halko2011finding}).  No further hyperparameter tuning is performed.

\paragraph{Experimental Protocol and Runs.}
Each “run” uses a different Pseudo Random Number Generator seed for the Gaussian initialization \(\bm G_0\).  We perform 10 runs per setting (for one given $\epsilon$ value, all methods are initialized with the same Gaussian matrix), varying only the PRNG seed across runs.

\paragraph{Confidence Intervals.}
We report the mean and a 99\% confidence interval over the 10 runs, estimated via nonparametric bootstrap (1\,000 resamples).

\paragraph{Computational Resources.}
All experiments ran on a 16-core CPU.  Figure~1 (10 runs × 14 \(\epsilon\) values × 5 methods) requires $\approx 30$ min of wall-clock time.

\section{Matrix-agnostic comparison of the runtime-dependent bounds}
\label{agnmatrcomp}
We analyze the tightness of our proposed runtime-dependent bound compared to the previous bound in a matrix-agnostic manner, focusing on the first iteration of the algorithm. 

At iteration $ l = 0 $, the leading singular vectors are initialized as $ \bm{X}^0 = \text{Q}(\bm{\Omega}) $, where $ \bm{\Omega} \sim \mathcal{N}(0,1)^{n \times k} $ and $ \text{Q}(\bm{\Omega}) $ denotes the orthonormal matrix obtained from the QR decomposition of $ \bm{\Omega} $. Since $ \bm{X}^0 $ is independent of the matrix $ \bm{A} $, it allows us to assess the relative tightness of the proposed bound $ \hat{\Delta}_0 $ compared to the previous bound $ \Delta_0^{prior} $ at the first step. We note that when the randomized power method runs for only one step, it is akin to sketching \citep{halko2011finding}.

Let $r(k, n)$ be the ratio of the expected values of the two bounds, as a function of the desired number of eigenvectors $k$ and the matrix dimension $n$, i.e. $ r(k, n) = \frac{\mathbb{E}[\Delta_0^{prior}]}{\mathbb{E}[\hat{\Delta}_0]}.$

\textbf{Theoretical approximation of $r(k,n)$.}
We derive an asymptotic approximation of $r(k,n)$ using Approximation~\ref{exp_upp_bounds}.

\begin{approximation}
    \label{exp_upp_bounds} Let $\bm{X}^0 = \text{ Q}(\bm{\Omega})$ where $\bm{\Omega} \sim \mathcal{N}(0,1)^{n \times k}$ and $\text{ Q}(\bm{\Omega})$ is the orthonormal matrix $\bm{Q}$ from the $QR$ decomposition of $\bm{\Omega}$. Let $\mu = \frac{n}{n-2}$ and variances $\sigma^2=\frac{2n^2 (n-1)}{(n-2)^2 (n-4)}$. We can approximate $\mathbb{E}[\hat{\Delta}_0]^2$ and $\mathbb{E}[\Delta_0^{prior}]^2$  as follows:
\begin{align}
    & \mathbb{E}[\Delta_l^{prior}]^2 &\approx \left(\sqrt{2\sigma^2 \cdot \log(kn)\frac{k^2 }{n^2}}  + \frac{k\mu}{n}\right) \text{ and }
     \mathbb{E}[\hat{\Delta}_l]^2 &\approx \left(\sqrt{2\sigma^2 \cdot \log(n)\frac{k}{n^2}}  + \frac{k\mu}{n}\right).
        \vspace{-5mm}
\end{align}
\end{approximation}
\label{sec:app_proof4}
We provide here the derivation for Approximation~\ref{exp_upp_bounds}:
\begin{proof}
    Let $\bm{M} \in \mathbb{R}^{n \times k}$ with i.i.d $\mathcal{N}(0,1)$-distributed entries. Let $\bm{M}=\bm{Q}\bm{R}$ be its QR factorization (by definition $\bm Q$ is orthogonal and $\bm R$ upper triangular). By the Barlett decomposition theorem \citep{muirhead2009aspects}, we know that $\bm{Q}$ is a random matrix distributed uniformly in the Stiefel manifold $\mathbb{V}_{k,n}$. Then by Theorem~2.2.1 of \citet{chikuse2012statistics}, we know that a random matrix $\bm{Q}$ uniformly distributed on $\mathbb{V}_{k,n}$ can be expressed as $\bm{Q} = \bm{Z}(\bm{Z}^\top \bm{Z})^{-\frac{1}{2}}$ with $\bm{Z}$ another matrix with i.i.d $\mathcal{N}(0,1)$-distributed entries. We approximate $\bm{Z}^\top \bm{Z}$ as a diagonal matrix and remark that its diagonal elements are distributed as random chi-squared variables with $n$ degrees of liberty. For a matrix $\bm{A}$, we denote by $\bm{A}^{\circ n}$ the elementwise Hadamard exponentiation. Then we have:
\begin{align}
    \bm{Q} _{ij}^2 &=  (\bm{Z}(\bm{Z}^\top \bm{Z})^{-\frac{1}{2}})_{ij}^2 \\
     &= (\bm{Z}^{\circ 2}((\bm{Z}^\top \bm{Z})^{-\frac{1}{2}})^{\circ 2} )_{ij} \\
    &= (\bm{Z}^{\circ 2}(\bm{Z}^\top \bm{Z})^{-1})_{ij}.
\end{align} 
We remark that each element of $\bm{Z}^{\circ 2}$ is distributed as a chi-squared variable with one degree of freedom, and therefore $n \cdot \bm{Q} _{ij}^2$ is distributed as an i.i.d $F(1, n)$ random variable by definition. For large $n$, we can approximate these $F$ variables as Gaussians with mean $\mu = \frac{n}{n-2}$ and variances $\sigma^2=\frac{2n^2 (n-1)}{(n-2)^2 (n-4)}$. Therefore $\bm{Q} _{ij}^2 \sim \mathcal{N}(\frac{\mu}{n},\frac{\sigma^2}{n^2})$ and:
\begin{align}
    & k \bm{Q} _{ij}^2 \sim \mathcal{N}(\frac{k \mu}{n},\frac{k^2 \sigma^2}{n^2})\,, && \sum_j \bm{Q} _{ij}^2 \sim \mathcal{N}(\frac{k\mu}{n},\frac{k\sigma^2}{n^2}).
\end{align}
We can approximate the expectation of the maximum (noted as m) of $d$ Gaussian variables distributed as $\mathcal{N}(\mu_2, \sigma_2^2)$ by $m = \sigma_2 \sqrt{2\cdot \log(d)} + \mu_2$ by Lemma 2.3 from \citet{massart2003concentration}. We therefore get approximations of $\mathbb{E}[\max_{ij} k\bm{Q} _{ij}^2]$ and $\mathbb{E}[\max_{i} \|\bm{Q} _{i:}\|^2]$. We assume that the variances of $\Delta_l^{prior}$ and $\hat{\Delta}_l$ are small (because $n$ and $k$ are large) and therefore:
\begin{align}
    & \mathbb{E}[\Delta_l^{prior}]^2 \approx \mathbb{E}[(\Delta_l^{prior})^2] = \sqrt{\frac{k^2 \sigma^2}{n^2}} \sqrt{2\cdot \log(kn)} + \frac{k\mu}{n}\,, \\
    & \mathbb{E}[\hat{\Delta}_l]^2 \approx \mathbb{E}[(\hat{\Delta}_l)^2] = \sqrt{\frac{k\sigma^2}{n^2}} \sqrt{2\cdot \log(n)} + \frac{k\mu}{n}.
\end{align}
\end{proof}

\paragraph{Empirical approximation of $r(k,n)$.}

To compare the tightness of our proposed bound $ \hat{\Delta}_0 $ to the previous bound $ \Delta_0^{prior} $, we perform both theoretical and empirical analyses at the first iteration, where the noise scaling depends only on the random initialization $ \bm{X}^0 $. Since $ \bm{X}^0 $ is independent of $ \bm{A} $, we can estimate $ \mathbb{E}[\Delta_0^{prior}] $ and $ \mathbb{E}[\hat{\Delta}_0] $ by sampling a random Gaussian matrix $\bm{X}^0$.

\noindent{\textit{Experiment:}} $r(k,n)$ depends on the number of factors $k$ and on $n$, where ($n \times n$) is the size of $\bm{A}$. We therefore seek to compute it for multiple values of $n$ ($n \in \{8000,12000\}$) and a range of values for $k$ (between 64 and 4000, with steps of 64), for each value of $n$. We use Approximation~\ref{exp_upp_bounds} to approximate it theoretically. To measure it empirically, we compute $\Delta_0^{prior}$ and $\hat{\Delta}_0$ based on $\bm{X}^0$ in the first step of the algorithm (it is independent of the matrix of interest $\bm{A}$), for $t$ runs of the algorithm. We can then use the $t$ measures to estimate $\mathbb{E}[\Delta_0^{prior}]$ and $\mathbb{E}[\hat{\Delta}_0]$. We claim that $t=5$ is enough to see a general trend, as we have estimates for many ($k,n$) couples.  We can also compare the empirical estimates to the proposed asymptotic approximations specified in Approximation \ref{exp_upp_bounds}. 

\begin{figure}[htb]
    \centering
    \subfigure[$n=8000$.]{
        \includegraphics[width=0.475\linewidth]{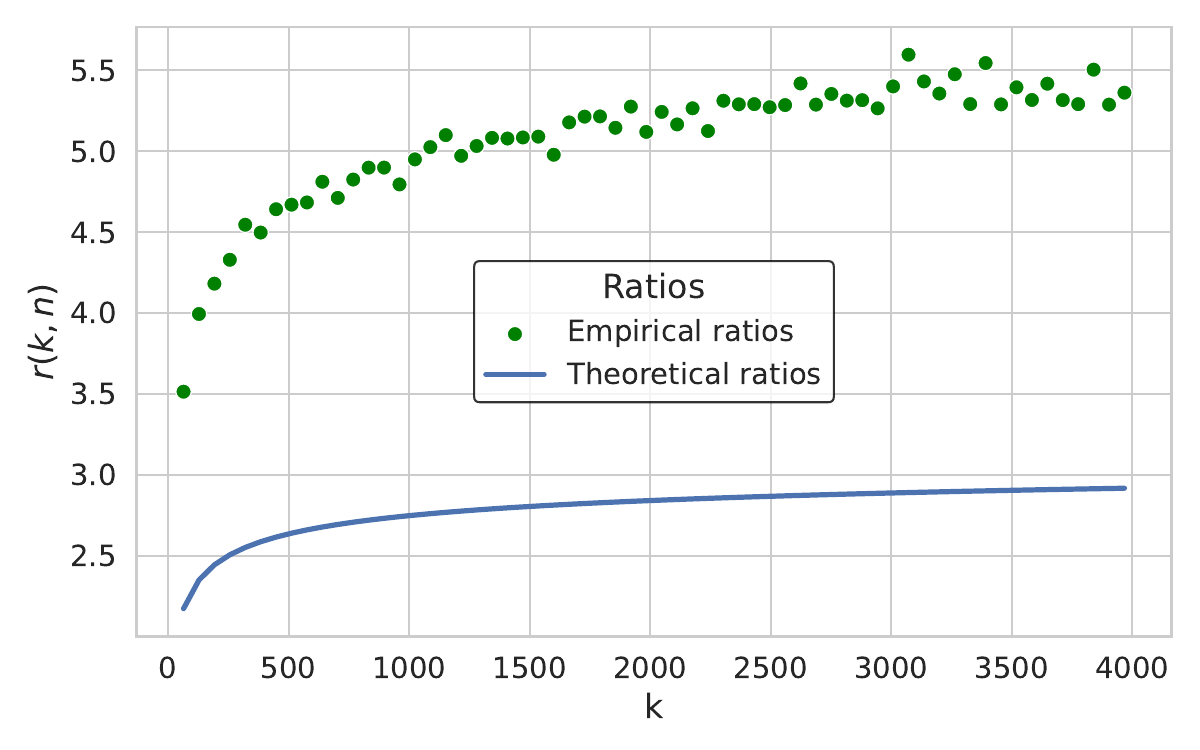}
        \label{fig:r_kn_8k}
        }
    \hfill
    \subfigure[$n=12000$.]{
        \includegraphics[width=0.475\linewidth]{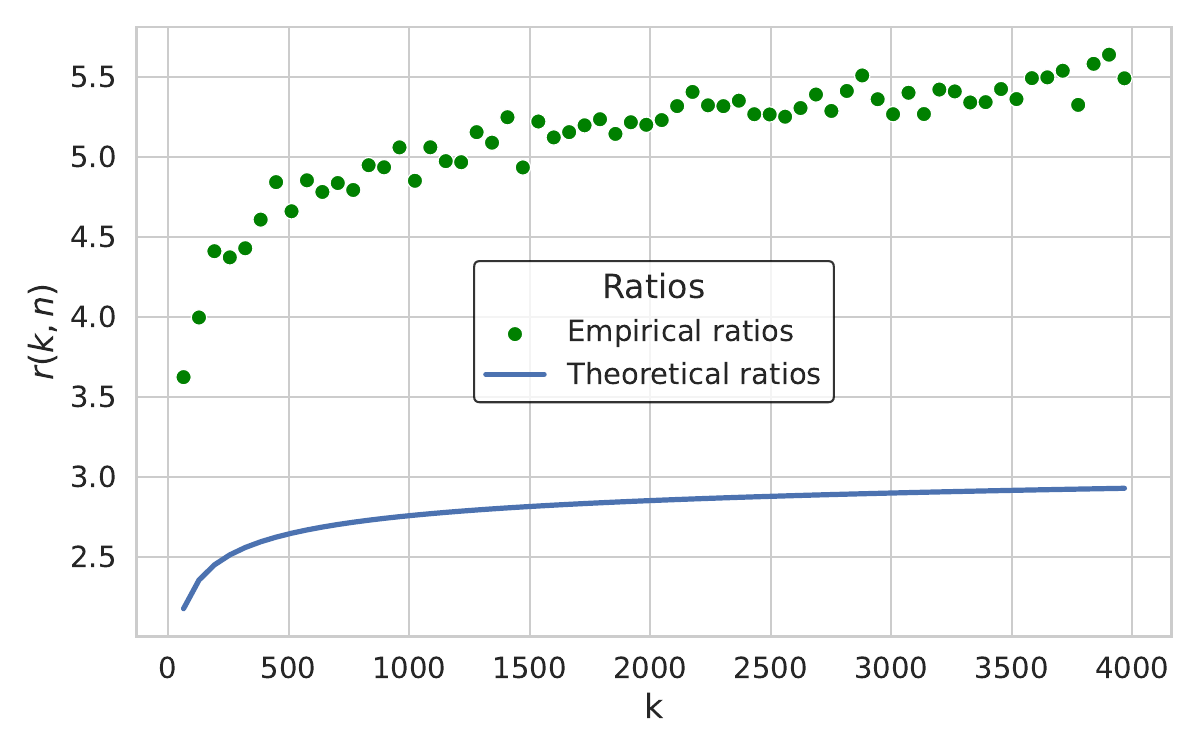}
        \label{fig:r_kn_12k}
    }
    \caption{Comparison of empirically (green dots) and theoretically (blue line) estimated $r(k, n)$ ratios for $k$ ranging between 64 and 4000 with a step of 64. Empirical ratios are estimated using $t=5$ runs of the first step of our algorithm, while theoretical ratios are based on Theorem~\ref{exp_upp_bounds}. Results are shown for two different values of $n$: 8000 and 12000.}
    \label{fig:r_kn_comparison}
\end{figure}

Figures~\ref{fig:r_kn_8k} and~\ref{fig:r_kn_12k} present the comparison between empirical (blue dots) and theoretical (green line) estimates of $r(k,n)$ for $n=8000$ and $n=12000$, respectively. The empirical ratios are calculated from the averages over the $t=5$ runs for each value of $k$. The results indicate that both empirically and theoretically the proposed noise scaling $\hat{\Delta}_l$ is much tighter than $\Delta_l^{prior}$ at the first step of the algorithm. Our theoretical approximation is conservative and underestimates how much tighter the bound is initially, compared to what we observe in practice. The proposed bound is tighter by a multiplicative factor on the first step and therefore drastically reduce the impact of the noise introduced by DP at the first iteration. By noting that the power method is usually run for very few steps ($L$ is usually in the range of $1$-$5$), this result complements our general convergence bounds on the overall algorithm derived in Theorems~\ref{thm:proposed_upper_bound} and  \ref{thm:ippmruninde} and gives further intuition on the tightness and usefulness of our proposed sensitivity bound.
\end{document}